\documentclass{article}

    \PassOptionsToPackage{numbers, compress}{natbib}

\usepackage[main,final]{neurips_2026}

\usepackage[utf8]{inputenc} 
\usepackage[T1]{fontenc}    
\usepackage{hyperref}       
\usepackage{url}            
\usepackage{booktabs}       
\usepackage{amsfonts}       
\usepackage{nicefrac}       
\usepackage{microtype}      
\usepackage{xcolor}         

\usepackage{amsmath}
\usepackage{amssymb}
\usepackage{mathtools}
\usepackage{amsthm}
\usepackage{algorithm}
\usepackage{algorithmic}
\usepackage{bigdelim}
\usepackage{pifont} 
\usepackage{tikz}
\usetikzlibrary{decorations.pathreplacing,calc}
\newcommand{\tikzmark}[1]{\tikz[overlay,remember picture] \node (#1) {};}
\definecolor{truegreen}{rgb}{0.1, 0.6, 0.1}
\definecolor{truered}{rgb}{0.8, 0.0, 0.0}
\definecolor{midamber}{rgb}{1.0, 0.49, 0.0}
\newcommand{\cmark}{\textcolor{truegreen}{\ding{51}}} 
\newcommand{\xmark}{\textcolor{truered}{\ding{55}}}   

\usepackage[capitalize,noabbrev]{cleveref}

\theoremstyle{plain}

\newtheorem{proposition}{Proposition}

\theoremstyle{definition}

\theoremstyle{remark}

\DeclareMathOperator*{\argmax}{arg\,max}

\newcommand{\n}{\mathrm{normal}}
\newcommand{\QMC}{\mathrm{QMC}}
\newcommand{\RQMC}{\mathrm{RQMC}}

\newcommand{\PM}{\mathrm{PM}}
\newcommand{\MVN}{\mathrm{MVN}}
\newcommand{\Poisson}{\mathrm{Poisson}}
\newcommand{\W}{W^{1/2}}
\newcommand{\z}{\hat{z}}
\newcommand{\trace}{\mathrm{trace}}

\newcommand{\SSS}{\mathcal{S}}

\newcommand{\T}{\mathcal{T}}
\newcommand{\Stu}{\mathrm{StudentT}}

\title{Corrected Integrated Laplace Approximation for\\ Bayesian Inference in Latent Gaussian Models}

\author{%
  Jinlin Lai\\
  Manning College of Information and Computer Sciences\\
  University of Massachusetts Amherst\\
  \texttt{jinlinlai@cs.umass.edu} \\
  \AND
  Charles C. Margossian \\
  Department of Statistics\\
  University of British Columbia \\
  \texttt{charles.margossian@ubc.ca} \\
  \AND
  Daniel Sheldon \\
  Manning College of Information and Computer Sciences\\
  University of Massachusetts Amherst\\
  \texttt{sheldon@cs.umass.edu} \\
}

\begin{document}

\maketitle

\begin{abstract}
Latent Gaussian models (LGMs) are a popular class of Bayesian hierarchical models that include Gaussian processes, as well as certain spatial models and mixed-effect models. Efficient Bayesian inference of LGMs often requires marginalizing out the latent variables. For LGMs with a non-Gaussian likelihood, exact marginalization is not possible and a popular approach is to do approximate marginalization with an integrated Laplace approximation (ILA). Using ILA produces an approximate posterior which, in some settings, can differ significantly from the correct posterior, which impacts downstream applications.
We propose an importance sampling scheme to correct the error introduced by ILA. By increasing the number of samples in importance sampling, the posterior with ILA converges to the correct posterior. This idea is realized with various techniques, including pseudo-marginalization, quasi-Monte Carlo and randomized quasi-Monte Carlo. 
We implement our methods in an automatic differentiation framework to support gradient-based algorithms when doing inference on the hyperparameters. For the latter, we specifically consider the use of Hamiltonian Monte Carlo. We demonstrate the benefits of reduced error in various applied models.
\end{abstract}

\section{Introduction}

Latent Gaussian models (LGMs) are a popular class of Bayesian models, formulated as a joint density $\pi(\theta, z, y) = \pi(\theta) \pi(z|\theta) \pi(y|\theta, z)$, where $\theta$ is often termed the ``hyperparameter'', $z$ is the latent Gaussian variable and $y$ is the observation.
In detail, the hierarchical prior $\pi(z|\theta)$ is Gaussian.
Examples of LGMs include Gaussian process models \citep{williams2006gaussian} and generalized linear models \citep{nelder1972generalized}.
Doing Bayesian inference on such models amounts to computing the joint posterior distribution $\pi(\theta, z|y)$.
This is a well-known challenge because $z$ is often high-dimensional and the posterior has the intricate geometry characteristic of hierarchical models \citep{betancourt2015hierarchy, papaspiliopoulos2007general}.

A common approach to tackle this problem is to marginalize out $z$ and then perform Bayesian inference on $\pi(\theta|y)$, a target with a lower dimension and often a simpler geometry.
Posterior samples for $z$ are then obtained in a post-inference step.
Marginalization can be done exactly when $\pi(y|\theta, z)$ is Gaussian, else it must be done \textit{approximately} with a Laplace approximation \citep{tierney1986accurate, williams2006gaussian, rue2009approximate}.
This idea underlies the popular \textit{integrated nested Laplace approximation} \citep[INLA,][]{rue2009approximate}.
INLA works effectively when $\pi(\theta|y)$ is well characterized around the mode and low-dimensional (``2--5, not more than 20'' \citep{rue2017inla}).
To overcome these limitations and scale to cases where $\theta$ is high-dimensional, it has been proposed to run Markov chain Monte Carlo (MCMC) over $\pi(\theta|y)$~\citep{kristensen2016tmb, monnahan2018no, margossian2020hamiltonian} and in particular gradient-based samplers, such as Hamiltonian  Monte Carlo \citep[HMC,][]{duane1987hybrid, neal2012hmc} which are known to scale favorably with dimension.
For this, \citet{margossian2020hamiltonian} developed \textit{adjoint-differentiated Laplace approximation} (ADLA), an automatic differentiation algorithm which efficiently computes the gradient of the approximate marginal log posterior, $\nabla_\theta \log \hat \pi(\theta|y)$, obtained with an integrated Laplace approximation.

Unfortunately, all the above described methods are subject to the approximation error introduced by the Laplace approximation.
An alternative is to simply run MCMC over the full posterior $\pi(\theta, z|y)$ but this approach, while asymptotically exact, is often computationally expensive, if not infeasible.
In this work, we develop methods that both benefit from marginalization with Laplace approximation and give an asymptotically correct posterior.

Our main idea is importance sampling. We show that the integrated Laplace approximation can be reformulated as an importance sampling estimator with one sample and the error of this estimator can be reduced with multiple samples. First, we propose ADLA using pseudo-marginalization~\citep{andrieu2009pseudo},
an approach which targets an extended distribution over the model parameters and auxiliary random variables used for importance sampling,
and we show that this approach is asymptotically correct.
Next, we propose to fix the randomness using quasi-Monte Carlo (QMC).
While QMC-based approaches remain asymptotically correct, we find that in practice they can have a large error, even when the number of samples used for importance sampling is large.
Finally, we unify pseudo-marginalization and QMC and propose ADLA with randomized QMC, an approach which utilizes the Laplace approximation and is asymptotically correct.
We implement the proposed methods in a reverse-mode automatic differentiation framework and test them across several applications.
We find that our proposed methods provide reduced error with the computational benefits of marginalization. The code is available at \url{https://github.com/lll6924/cila.git}.


\section{Background}
\subsection{Latent Gaussian models}
LGMs observe the following hierarchical structure,
\begin{align}
    \theta\sim \pi(\theta),\ z\sim \n(0,K(\theta)),\ y\sim \pi(y|\theta,z),
\end{align}
where $\theta$ is the hyperparameter, $z$ comprises the latent Gaussian effects, and $y$ is the observation. 
Here, the prior Gaussian on $z$ is centered at 0 but off-set from 0 can be encoded in the likelihood $\pi(y|\theta, z)$.
Without loss of generality, we split $\theta$ into two, $\theta=[\xi,\eta]$, where $\xi$ parametrizes the prior covariance $K$, meaning $K(\theta) = K(\xi)$, and $\eta$ parametrizes the likelihood, meaning $\pi(y|\theta, z)=\pi(y|\eta,z)$. 


\subsection{Adjoint-differentiated Laplace Approximation}
A general way of approximately marginalizing $z$ from LGMs is the Laplace approximation. In detail, the Laplace approximation approximates the conditional posterior $\pi(z|\theta,y)$ by a Gaussian, which matches the mode and curvature of the conditional posterior,
\begin{align}
\label{eq:la_mode_curvature}
    \z=\argmax_{z}\log \pi(y,z|\theta),\quad \Sigma^{-1}=-\nabla_z\nabla_z\log \pi(y,\z|\theta), 
\end{align}
where $\z$ can be obtained from any optimization algorithm (Newton's method is often used in practice).
Then $\pi(z|\theta,y)$ is approximated by $\hat{\pi}(z|\theta,y)=\n(\z, \Sigma)$. 
The approximation is exact when the likelihood $\pi(y|\theta,z)$ is normal and so, by conjugacy, $\pi(z|\theta, y)$ is also normal. 
In practice, this approximation is well justified when $\pi(y|\theta,z)$ is log-concave \citep{shun95taylor}, although it has also been explored in other contexts \citep[e.g.,][]{vanhatalo09gp}.

For any fixed $z$, the marginalized model can be approximated by
\begin{align}
\label{eq:la_approximation_mode}
\hat{\pi}_{z}(\theta,y)\coloneq\frac{\pi(\theta)\pi(z|\theta)\pi(y|\theta,z)}{\hat{\pi}(z|\theta,y)}\approx \frac{\pi(\theta)\pi(z|\theta)\pi(y|\theta,z)}{\pi(z|\theta,y)}=\pi(\theta,y).
\end{align}
This approximation is valid for any $z$, but $z=\z$ from the optimizer is usually used. Note that $\hat{z}=\hat{z}(\theta)$ is a function of $\theta$.
We omit this dependence here for simplicity, but it is important when computing the gradients.
%
%
%

With approximation (\ref{eq:la_approximation_mode}), we can run inference on a lower-dimensional model $\hat{\pi}_{\z}(\theta,y)$ and get an approximate marginal posterior $\hat{\pi}_{\z}(\theta|y)$. In this work, we focus on Hamiltonian Monte Carlo (HMC)~\citep{duane1987hybrid, neal2012hmc}, an MCMC sampler that scales well in high-dimensions but requires computing $\nabla_\theta \log \hat{\pi}_{\z}(\theta,y)$. 
The calculation of this gradient must be done carefully: in particular, we need to account for the fact that $\z$ depends implicitly on $\theta$ and that the Laplace approximation is obtained using second-order derivatives of $\log \pi(z|y, \theta)$ to compute $\Sigma$.
Calculating the gradient can be done efficiently using the implicit function theorem and adjoint methods of automatic differentiation, as in the adjoint-differentiated Laplace approximation (ADLA)~\citep{margossian2020hamiltonian,margossian2023general}.
Appendix~\ref{sec:adla_detail} provides details on ADLA.
Here, we implement ADLA in JAX~\citep{jax2018github}, a high-performance automatic differentiation library.


HMC generates asymptotically unbiased samples from the approximate marginal posterior $\hat{\pi}_{\z}(\theta|y)$. 
When the approximation in Eq. (\ref{eq:la_approximation_mode}) is not exact, the posterior approximation from HMC differs from the true posterior $\pi(\theta|y)$, even asymptotically. We call this issue `posterior error' and propose to correct the error by closing the gap in Eq. (\ref{eq:la_approximation_mode}).
In this paper, we construct approximate marginal models with importance sampling that better approximate $\pi(\theta,y)$, and then derive an algorithm to compute gradients.

\begin{table*}[t]
\centering
\caption{Comparison of different methods of marginalization on posterior correctness, dimension change and continuity of the objective. }
\label{tab:alg_comparison}
\begin{tabular}{lllll}
\toprule
\textbf{Algorithm}  & Correct posterior & Dimension & Continuity\\ \midrule
ADLA~\cite{margossian2023general} &\xmark&\textcolor{truegreen}{Lower}&\cmark\\
PM-ADLA (Sec. \ref{sec:pm}) &\cmark&\textcolor{truered}{Possibly Higher}&\cmark\\
QMC-ADLA (Sec. \ref{sec:qmc}) &\cmark (Asymptotically)&\textcolor{truegreen}{Lower}&\cmark\\
RQMC-ADLA (Sec. \ref{sec:rqmc}) &\cmark&\textcolor{black}{Same}&\xmark\\
\bottomrule
\end{tabular}
\end{table*}
\section{Reducing error with importance sampling}

We propose an unbiased estimator of the marginal model $\pi(\theta, y)$.
This estimator is similar to the approximation in Eq.~\eqref{eq:la_approximation_mode}, except that the Laplace approximation is not evaluated at the mode $\z$ but at a random point $z$ drawn from the Laplace approximation.
\begin{proposition}
\label{prop:1}
    Consider the Laplace approximation $\hat \pi(z|\theta, y)$ in Eq.~\eqref{eq:la_mode_curvature} and let
    \begin{equation} \label{eq:laplace_approximation_z}
    \hat \pi_z(\theta, y) = \frac{\pi(\theta)\pi(z|\theta)\pi(y|\theta,z)}{\hat{\pi}(z|\theta,y)},
    \end{equation}
    where $z \sim \hat \pi(z|\theta, y)$.
    Then $\hat \pi_z(\theta, y)$ is an unbiased estimator of $\pi(\theta, y)$.
\end{proposition}
\begin{proof}
  Observe that,
  \begin{align}
\hat{\pi}_z(\theta,y)=\pi(\theta,y)\frac{\pi(z|\theta,y)}{\hat{\pi}(z|\theta,y)}.\notag
\end{align}
Here, we may recognize that $\hat \pi_z(\theta, y)$ is an importance sampling estimator of $\pi(\theta, y)$ and so must be unbiased.
In detail,
\begin{equation}
    \mathbb E [\hat \pi_z(\theta, y)]= \int \pi(\theta,y)\frac{\pi(z|\theta,y)}{\hat{\pi}(z|\theta,y)} \hat{\pi}(z|\theta,y) \text d z = \pi(\theta, y) \int \pi(z|\theta,y) \text d z = \pi(\theta, y).\notag
\end{equation}
\end{proof}



%
If we have $n$ i.i.d. samples, $z_1,\ldots, z_n\sim \hat{\pi}(z|\theta,y)$, we can construct an averaged estimator for the true marginal
\begin{align}
\hat{\pi}_{z_{1:n}}(\theta,y)=\frac{1}{n}\sum_{i=1}^n\hat{\pi}_{z_i}(\theta,y).    \notag
\end{align}
This estimator is also unbiased from linearity of expectation, and its variance is of order $\mathcal{O}(n^{-1})$ if the variance of the importance sampling weights is finite. 

In order to use our unbiased estimator of the marginal posterior in HMC, we must overcome two challenges.
First, standard HMC only works for a fixed target density, rather than a target density that varies with an auxiliary random variable $z_i$.
In the next two sections, we consider different ways to handle this random target density.
Second, we must generalize the automatic differentiation algorithm used by ADLA to compute $\nabla_{\theta}\log \hat{\pi}_{z_{1:n}}(\theta,y)$.
Indeed, in the classical setting of the integrated Laplace approximation, it suffices to propagate derivatives through $\z$.
Here, we must propagate derivatives through $z_i$.
We detail this procedure in Appendix~\ref{sec:adla_detail}.

\section{Pseudo-marginal ADLA}
\label{sec:pm}
To build a fixed approximate marginal model, our first idea is to target an extended distribution that includes the auxiliary random variables $z_i$ introduced to construct the approximate marginal model in Eq.~\eqref{eq:laplace_approximation_z}.
This approach is known as pseudo-marginalization in the literature~\citep{andrieu2009pseudo}. In detail, we express each sample $z_i$ as a transformation over some random auxiliary variables $\epsilon_i$,
\begin{align}
    z_i=\T_{\theta,y}(\epsilon_i)\coloneqq\z+\sqrt{\Sigma}\epsilon_i,\ \text{where }\epsilon_i\sim\pi(\epsilon_i)=\n(0,I),\notag
\end{align}
for $i=1,2,...,n$. Here $\sqrt{\Sigma}$ is a matrix square root such that $\sqrt{\Sigma}\sqrt{\Sigma}^T=\Sigma$, so each $z_i$ still follows $\hat{\pi}(z|\theta,y)=\n(\z,\Sigma)$. Then, our approximate marginal model is the importance sampling estimator multiplied by the density of the auxiliary variables:
\begin{align} \label{eq:laplace_pm}
    \hat{\pi}^{\PM}(\theta,\epsilon_{1:n},y)\coloneq \prod_{i=1}^n\pi(\epsilon_i)\left(\frac{1}{n}\sum_{i=1}^n\frac{\pi(\theta)\pi(\T_{\theta,y}(\epsilon_i)|\theta)\pi(y|\theta,\T_{\theta,y}(\epsilon_i))}{\hat{\pi}(\T_{\theta,y}(\epsilon_i)|\theta,y)}\right),
\end{align}
where we have replaced $z_i$ with $\T_{\theta,y}(\epsilon_i)$ when computing the density. This approximate marginal model has several nice properties. First, it shares the same marginal over $(\theta,y)$ as the original model.
\begin{proposition}
  $\hat{\pi}^{\PM}$, as defined in eq.~\eqref{eq:laplace_pm}, admits $\pi(\theta, y)$ as its marginal.
\end{proposition}
The proof is in Appendix~\ref{proof:pm}. By sampling from the extended posterior $\hat{\pi}^{\PM}(\theta,\epsilon_{1:n}|y)$, we also obtain samples from the marginal posterior $\pi(\theta|y)$. 
Here, the Laplace approximation does the marginalization in a ``pseudo'' way: the latent Gaussian is marginalized out but we augment the state space with additional noise variables $\epsilon$. We call this approach PM-ADLA.

We now argue that this model offers the benefits of the Laplace approximation and of importance sampling.
To see this, write
\begin{align}
    \hat{\pi}^{\PM}(\theta,\epsilon_{1:n},y)= \pi(\theta,y)\prod_{i=1}^n\pi(\epsilon_i)\left(\frac{1}{n}\sum_{i=1}^n\frac{\pi(\T_{\theta,y}(\epsilon_i)|\theta,y)}{\hat{\pi}(\T_{\theta,y}(\epsilon_i)|\theta,y)}\right).\notag
\end{align}
When $n$ is large, the last term $n^{-1} \sum_{i=1}^n\pi(\T_{\theta,y}(\epsilon_i)|\theta,y) / \hat{\pi}(\T_{\theta,y}(\epsilon_i)|\theta,y)$ approaches $1$ and the model is the true marginal $\pi(\theta,y)$ multiplied by $n$ independent normal distributions.
Then, the augmented model does not possess the intricate geometry of the original posterior $\pi(\theta, z|y)$, since $\epsilon$ is independent of $\theta$. 

Compared with ADLA, PM-ADLA corrects the posterior error, but does not reduce the problem dimension. Instead, if the dimension of $z$ is $d_z$, the model dimension even increases by $(n-1)d_z$ for $n>1$. In addition, at each step of HMC, PM-ADLA requires $n$ evaluations of the model density and its gradient.
(However, we still only need to compute the Laplace approximation once.)

\section{ADLA with quasi-Monte Carlo}
\label{sec:qmc}
Instead of introducing auxiliary random variables and in doing so increasing the model dimension, we can instead fix those variables. A naive approach is to sample the auxiliary variables $\epsilon_1,\ldots,\epsilon_n$ independently once in the beginning and reuse them during inference to ensure a fixed approximate model. The approximate marginal model is
\begin{align}
    \hat{\pi}_{\epsilon_{1:n}}^{\mathrm{IS}}(\theta,y)\coloneq \frac{1}{n}\sum_{i=1}^n\frac{\pi(\theta)\pi(\T_{\theta,y}(\epsilon_i)|\theta)\pi(y|\theta,\T_{\theta,y}(\epsilon_i))}{\hat{\pi}(\T_{\theta,y}(\epsilon_i)|\theta,y)}.\notag
\end{align}
When the estimation variance of importance sampling is $\mathcal{O}(n^{-1})$, this converges to the correct posterior density as $n\to\infty$. 

A more efficient way is to sample correlated auxiliary variables to further reduce the estimation variance of importance sampling. Intuitively, we want the randomness to spread more uniformly than independent sampling. To achieve this, we can use a low-discrepancy sequence (LDS) in the space of the unit cube as the auxiliary variables. The approach, called quasi-Monte Carlo (QMC)~\citep{caflisch1998monte}, further reduces the estimation variance of importance sampling to $\mathcal{O}(n^{-2})$ under proper conditions. We demonstrate that QMC works better than raw importance sampling for correcting the error in Appendix \ref{sec:raw_is}. In this work, we use the Sobol sequence, but other sequences, such as the Halton sequence, could also be applied. To use auxiliary variables in the unit cube, we change our reparameterization to,
\begin{align}
z_i=\T_{\theta,y}(u_i)=\z+\sqrt{\Sigma}\Phi^{-1}(u_i),\ \text{where }u_i\sim \mathrm{uniform}([0,1]^{d_z}),\notag
\end{align}
for $i=1,\ldots,n$. Here $\Phi^{-1}$ is the inverse-CDF function of the unit normal distribution, so $z_i\sim \hat{\pi}(z|\theta,y)$, ensuring our estimator is a valid importance sampling estimator. With this parameterization, we fix an LDS $u_1,\ldots u_n$ and denote the QMC-based approximate marginal density by
\begin{align}\label{eq:laplace_qmc}
    \hat{\pi}_{u_{1:n}}^{\QMC}(\theta,y)\coloneq \frac{1}{n}\sum_{i=1}^n\frac{\pi(\theta)\pi(\T_{\theta,y}(u_i)|\theta)\pi(y|\theta,\T_{\theta,y}(u_i))}{\hat{\pi}(\T_{\theta,y}(u_i)|\theta,y)}.
\end{align}
We call this method QMC-ADLA. The approximate marginal distribution with the original ADLA can be viewed as QMC-ADLA where each $u_i=1/2$, making $z_i=\z$.

$\hat{\pi}_{u_{1:n}}^{\QMC}(\theta,y)$ defines an approximate posterior $\hat{\pi}_{u_{1:n}}^{\QMC}(\theta|y)$.
The next proposition shows that the total variation distance between $\hat{\pi}_{u_{1:n}}^{\QMC}(\theta,y)$ and $\pi(\theta|y)$ vanishes as $n$ increases.
\begin{proposition}
Let $f(u)=\frac{\pi(\theta)\pi(\T_{\theta,y}(u)|\theta)\pi(y|\theta,\T_{\theta,y}(u))}{\hat{\pi}(\T_{\theta,y}(u)|\theta,y)}$. If $f(u)$ has bounded variation~\citep{owen2005multidimensional}, and there exists a function $g(\theta)$ such that $\hat{\pi}_{u_{1:n}}^{\QMC}(\theta,y)<g(\theta)$ and $\int g(\theta)d\theta<\infty$, then $\int \left|\hat{\pi}_{u_{1:n}}^{\QMC}(\theta|y)-\pi(\theta|y)\right|d\theta=0$ as $n\to\infty$,
\end{proposition} 

The proof is in Appendix~\ref{proof:qmc_converge}. 
Note that the condition of bounded variation is required for the convergence rate of QMC but is difficult to verify and may not hold in practice. Our framework allows $\hat{\pi}(z|\theta,y)$ to be replaced with any approximation of $\pi(z|\theta,y)$, including those with better tail behaviors. This flexibility may allow the conditions of QMC to be met more easily.

Naturally, for finite $n$, some posterior error still persists, and in our experiments, we find cases where the error is large even when $n=64$. Also, similar to PM-ADLA, QMC-ADLA requires $n$ density evaluations at each step of HMC.

\subsection{ADLA with randomized QMC}
\label{sec:rqmc}
It is possible to combine QMC and PM to construct an approximate marginal model that converges to the true marginal and does not increase the model dimension. The technique is randomized QMC (RQMC)~\citep{owen2000monte}, which introduces a single shift variable $U\sim \mathrm{uniform}([0,1]^{d_z})$ as the only auxiliary variable.
Next, let
\begin{align}
    u_i=(v_i+U)\% 1\notag
\end{align}
for $i=1,\ldots, n$, where ``$\%$'' is the modulo operator and $v_1,\ldots,v_n$ is a \emph{fixed} LDS in the space of $[0,1]^{d_z}$. To simplify our notation, we can define the mapping from $U$ to each $z_i$ by 
\begin{align}
    z_i=\SSS_i(U)\coloneq \T_{\theta,y}((v_i+U)\% 1).\notag
\end{align}
Then, the RQMC-based approximate marginal model is
\begin{align} \label{eq:laplace_rqmc}
    \hat{\pi}^{\RQMC}(\theta,U,y)\coloneq \pi(U)\left(\frac{1}{n}\sum_{i=1}^n\frac{\pi(\theta)\pi(\SSS_i(U)|\theta)\pi(y|\theta,\SSS_i(U))}{\hat{\pi}(\SSS_i(U)|\theta,y)}\right).
\end{align}
We call this estimator RQMC-ADLA, which is also correct for marginalization according to the following proposition. 
\begin{proposition}
  $\hat{\pi}^{\RQMC}$, as defined in eq.~\eqref{eq:laplace_rqmc}, admits $\pi(\theta, y)$ as its marginal.
\end{proposition}
The proof is in Appendix~\ref{proof:rqmc}. Similar to PM-ADLA, as $n\to\infty$, the model converges to the true marginal multiplied by an independent uniform distribution over $U$. Since $U$ has the same dimension as $z$, RQMC-ADLA does not increase the model dimension, and can be viewed as a tunable reparameterization with the hyperparameter $n$. 

However, unlike previous approaches, the function $\SSS_i(U)$ is not continuous due to modulo operators, so the density $\hat{\pi}^{\RQMC}(\theta,U,y)$ is not continuous in $U$ and HMC is not directly applicable. Fortunately, given a fixed $U$, $\hat{\pi}^{\RQMC}(\theta,U,y)$ is continuous in $\theta$, so it is possible to use a Metropolis-within-Gibbs (MwG) sampler~\citep{gilks1995adaptive}. In each round, we update $\theta$ conditional on $U$ with HMC, and then update $U$ conditional on $\theta$ with Metropolis-Hastings. Details of the approach can be found in Appendix~\ref{sec:mwg}. 

We summarize the high-level trade-offs of all variants of ADLA in Table \ref{tab:alg_comparison}. 

Now we analyze the time complexity. With HMC, in every leapfrog step, we need to compute the model density once. If we use ADLA, an additional $\mathcal{O}(d_z^3)$ complexity is introduced for the Laplace approximation. Our methods are more expensive than ADLA or HMC. To construct our approximate marginal posterior, the importance sampling estimator requires $n$ copies of the latent variables $z_1,\ldots, z_n$, which implies computing the original model density $n$ times. Furthermore, the reparameterization from the auxiliary variables $\epsilon_1,\ldots,\epsilon_n$ to $z_1,\ldots, z_n$ takes $\mathcal{O}(nd_z^2)$ time. These additional computations occur in every leapfrog step. In RQMC-ADLA, we also have a Gibbs sampling step in addition to the HMC step. In the Gibbs sampling step, the $d_z$-dimensional variable $U$ is updated one coordinate at a time. To update each coordinate, we need $\mathcal{O}(nd_z^2)$ time for reparameterization to get $z_1,\ldots,z_n$ plus $n$ density computations. Therefore, RQMC-ADLA requires $\mathcal{O}(nd_z^3)$ time complexity and $nd_z$ density computations for Gibbs sampling. This implies that RQMC-ADLA scales more poorly with $n$ or $d_z$ compared with other methods. Although our methods are more expensive per step, this could be offset by a higher effective sample size per step, which eventually leads to more efficient sampling.


\section{Recovering marginalized variables}
In ADLA, the posterior distribution of $z$ can be constructed using $\hat{\pi}(z|\theta,y)$. For example, within HMC, for each sample of $\theta$, we generate $z\sim \hat{\pi}(z|\theta,y)$ such that $(\theta,z)\sim \hat{\pi}(\theta,z|y)\approx \pi(\theta,z|y)$. We can use the same procedure in our methods. However, sampling from the Laplace approximation introduces an error. With importance sampling, we can get posterior samples of $z$ that are closer to the true posterior $\pi(z|\theta,y)$. 

In any of our methods, for each $\theta$, we have $n$ unnormalized importance weights
\begin{align}
    w_i\coloneq \frac{\pi(\theta)\pi(z_i|\theta)\pi(y|\theta,z_i)}{\hat{\pi}(z_i|\theta,y)},\notag
\end{align}
for $i=1,\ldots,n$, and also $n$ samples $z_1,\ldots, z_n$. Then, we let
\begin{align}
    z\sim \hat{\pi}(z|z_{1:n},\theta,y):=\frac{\sum_{i=1}^n w_i\delta_{z_i}(z)}{\sum_{i=1}^nw_i}.\notag
\end{align}
The distribution of $z$ then matches the true distribution $\pi(z|\theta,y)$ by the following proposition.
\begin{proposition}
    If for each $i$, $z_i\sim \hat{\pi}(\cdot|\theta,y)$, let $\hat{\pi}(\theta,z_{1:n},y)=\left(\frac{1}{n}\sum_{i=1}^nw_i\right) \hat{\pi}(z_{1:n}|\theta,y)$, then 
    \begin{align}
        \hat{\pi}(\theta,z,y)\coloneq\int  \hat{\pi}(z|z_{1:n},\theta,y)\hat{\pi}(\theta,z_{1:n},y)dz_{1:n}=\pi(\theta,z,y).\notag
    \end{align}
\end{proposition}
The proof is in Appendix \ref{proof:recover}. In the proposition, $\hat{\pi}(\theta,z_{1:n},y)$ is the importance sampling estimator multiplied by the density of $z_1,\ldots z_n$, equivalent to the approximate marginal model in Eq. \ref{eq:laplace_pm} with a change of variable. Coupled with the recovery distribution $\hat{\pi}(z|z_{1:n},\theta,y)$, we get a joint model $\hat{\pi}(\theta, z_{1:n},z,y)$ whose marginal is the model $\pi(\theta,z,y)$. For PM-ADLA and RQMC-ADLA, since each $z_i$ follows $\hat{\pi}(\cdot|\theta,y)$, we have $\hat{\pi}(z|\theta,y)=\pi(z|\theta,y)$. At each step of sampling, we can recover $z$ directly from already computed intermediate results. For QMC-ADLA, we can follow the same procedure, but the recovery becomes approximate. This derivation is related to importance weighted variational inference~\citep{domke2018importance, domke2019divide}. 

\section{Related Work}

The integrated Laplace approximation plays an important role in the influential software packages INLA~\citep{rue2009approximate}, TMB~\citep{kristensen2016tmb}, and GPstuff~\citep{vanhatalo2013gpstuff}.
TMB's inverse subset algorithm \citep{kristensen2016tmb} and ADLA \citep{margossian2020hamiltonian, margossian2023general} propose implementations of the integrated Laplace approximation in probabilistic programming languages which use automatic differentiation, respectively TMB and Stan~\citep{carpenter2017stan}.
Within these frameworks, users have a great deal of flexibility when specifying their model, however the accuracy of the Laplace approximation is not guaranteed to hold for an arbitrary model.
To truly take advantage of the flexibility afforded by probabilistic programming languages, users need broadly applicable guarantees such as the ones we obtain for our methods.

Several papers have proposed replacing the integrated Laplace approximation with a more accurate approximation, for example using a copula-based correction \citep{ferkingstad15copula} or a higher-order Taylor expansion of $\log \pi(z|\theta, y)$ \citep{shun95taylor, chiuchiolo22inla}.
These methods can readily be combined with our proposed estimators, since our estimators (eqs.~\ref{eq:laplace_pm}, \ref{eq:laplace_qmc}, and \ref{eq:laplace_rqmc}) do not in fact rely on $\pi(z|\theta, y)$ being approximated by a standard Laplace approximation.
(On the other hand, designing differentiation algorithms for new approximations would require additional engineering effort.)


There is a rich literature on the use of importance sampling for approximate Bayesian inference.
\citet{berild22inla} combine importance sampling and the integrated Laplace approximation.
Specifically, importance sampling is used as a backend inference algorithm (instead of MCMC) to approximate the marginal posterior $\pi(\theta|y)$; by contrast, we construct an importance sampling estimator of the marginal joint $\pi(\theta, y)$.
The two methods can be combined and more generally, any (gradient-based) inference method can be used with our proposed estimators of $\pi(\theta, y)$.
Importance sampling has also been used to improve posterior approximation of variational inference \citep{yao18vi, domke2018importance}.
\citet{domke2019divide} generalize the idea and discuss QMC-based variational objectives. The recovery step of our methods is closely related to these works. There is another line of work that fixes the randomness of importance sampling-based variational objectives~\citep{pmlr-v244-burroni24a, giordano2024black}.
Finally, \citet{alenlov2021pseudo} explore using importance sampling within pseudo-marginalization for HMC. 
Different from all above works, the importance sampling in our work is applied to only part of the model, and uses the Laplace approximation as the proposal distribution. 

Our work can also be related to recent studies on learning the best reparametererization of Bayesian models~\citep{gorinova2020automatic, parno2018transport, hoffman2019neutra, komodel}. RQMC-ADLA can be viewed as a training-free reparameterization with convergence guarantees.

\section{Experiments}
Our methods are implemented in BlackJAX~\citep{cabezas2024blackjax}, using the default No-U-turn sampler (NUTS)~\citep{hoffman2014no} of HMC in NumPyro~\citep{bingham2019pyro,phan2019composable} as the inference algorithm. We demonstrate the performance of our methods with synthetic and real examples. All experiments are conducted on an Intel Xeon Platinum 8352Y CPU. In this section, we use `Base' to represent HMC without marginalization, and
 call our methods `PM' (Section~\ref{sec:pm}), `QMC' (Section~\ref{sec:qmc}) and `RQMC' (Section~\ref{sec:rqmc}).
When computing the estimation error, we aggregate the results of all `Base' runs as the ground-truth. 

\subsection{Gaussian process}
\begin{figure*}[t]
\centering
    \includegraphics[width=0.99\linewidth]{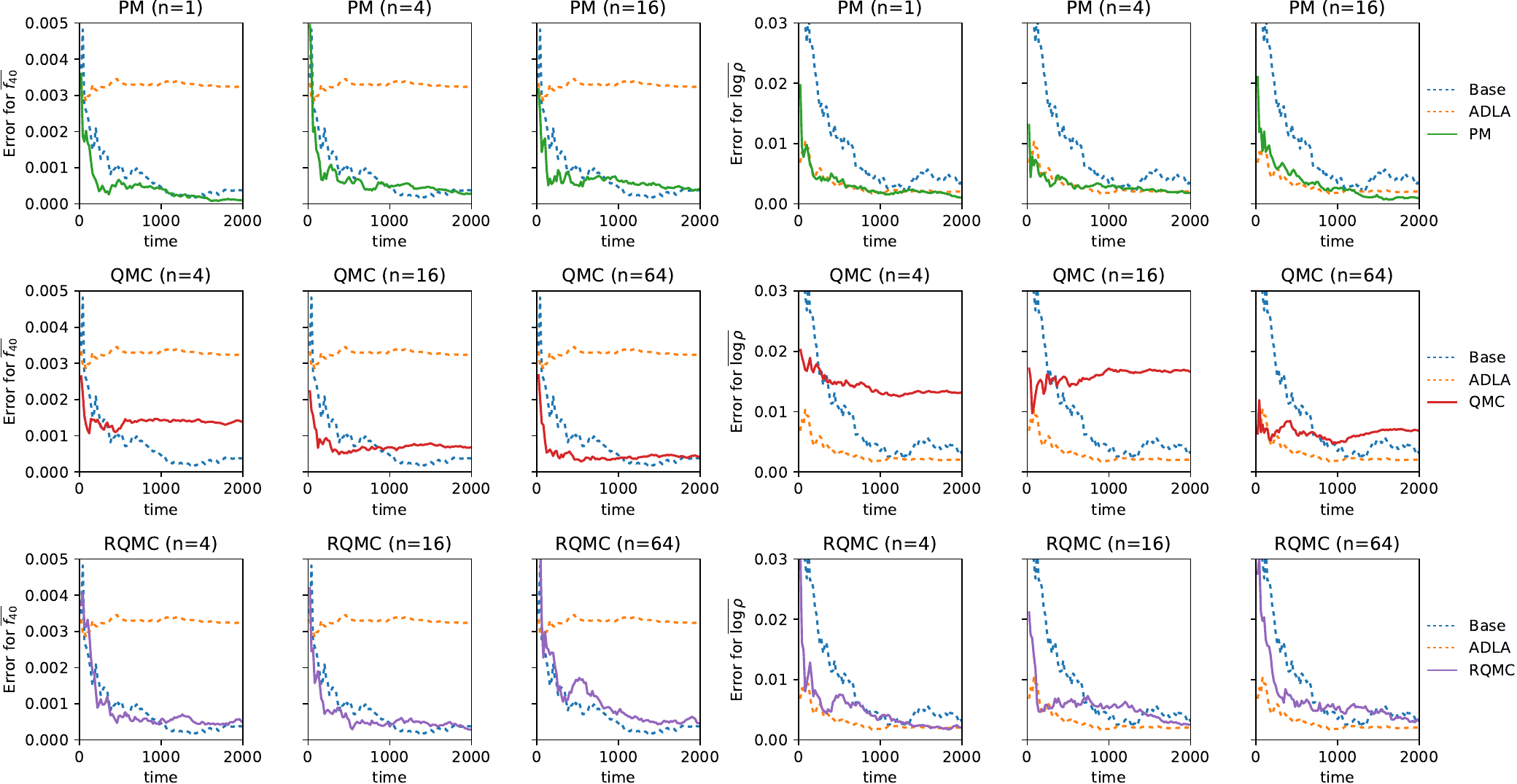}

    \caption{Error of estimating the means of parameters as a function of time in seconds for the synthesized Gaussian process with Poisson likelihood. Results are averaged from 5 independent runs. Ground-truth is estimated from NUTS on the unmarginalized model.}
        \label{fig:gppoisson_error}
\end{figure*}
A widely used family of LGMs is Gaussian process models. We synthesized a dataset of features $X\in\mathbb{R}^{N\times d}$ and responses $y\in \mathbb{N}^{N}$, where $N=100$ and $d=1$, and modeled these data with a log-Gaussian Cox process:
\begin{gather}
    \rho\sim\Gamma^{-1}(3,2),\ \alpha\sim\mathrm{Exponential}(1),\ f\sim \MVN(0,K(X,\rho,\alpha)),\ y\sim \Poisson(e^f),\notag
\end{gather}
where $K_{ij}(X,\rho,\alpha)=\alpha^2\exp\left(-\frac{1}{2\rho^2}(x_i-x_j)^2\right)$. Details can be found in the Appendix. The Laplace approximation can be applied to marginalize out the latent variables $f$. 

This model is difficult to sample with HMC directly. We find that, even with a non-centered parameterization~\citep{papaspiliopoulos2007general}, there are still on average 205 (out of 100,000) divergent transitions, which indicate occasional numerical instability due to the geometry of the posterior \citep{betancourt2015hierarchy}. Getting full HMC to work also requires increasing the target acceptance probability from a default value of 0.8 to 0.99. 
In contrast, ADLA and our methods all report zero divergences with the default target acceptance probability without any reparameterization, meaning the \textit{approximate} marginal posterior has a well-behaved geometry.
We then use the samples to estimate the mean of the model parameters. 

In Figure \ref{fig:gppoisson_error}, we plot the estimation error as a function of time. 
When estimating $\mathbb{E}[\log\rho]$, ADLA has a lower error than HMC throughout the period. 
However, ADLA's estimation error for $\mathbb{E}[f_{40}]$ is evident from the beginning. 
In contrast, our proposed methods enjoy the speed-up of ADLA and reduced error as sampling proceeds. In the second row, we find that QMC gives better estimates for $\mathbb{E}[f_{40}]$ but worse estimates for $\mathbb{E}[\log\rho]$, and performs better with larger $n$. In the first and third rows, we see that both PM and RQMC have less error than ADLA for estimating $\mathbb{E}[f_{40}]$, and faster inference than HMC for estimating $\mathbb{E}[\log\rho]$. Despite using multiple samples, we find that QMC's bias may be higher than ADLA. When PM or RQMC is applicable, we find it is possible to correct the error of ADLA with comparable convergence speed.

In Figure \ref{fig:gppoisson}, we further compare the running time and the effective sample size (ESS) per minute of different methods.
An important caveat with this metric is that the ESS characterizes the variance of the Monte Carlo estimator \textit{but not its bias}.
A high ESS/min can be interpreted as HMC efficiently sampling from the approximate posterior.
Only if this approximation is accurate, as verified in Figure~\ref{fig:gppoisson_error}, can we further interpret a high ESS/min as efficient sampling from the correct posterior. 

ADLA and QMC run fastest but can suffer from a high bias for certain parameters. 
%
%
We find that, although PM is multiple times slower than ADLA, the fast mixing of the MCMC chains leads to a high ESS/min.
Furthermore, PM produces a low error across parameters.
RQMC is also slow because each Gibbs step requires multiple additional evaluations of the density function, which leads to the lowest ESS/min among the marginalization methods.

\begin{figure*}[t]
    \centering
    \includegraphics[width=0.4\linewidth]{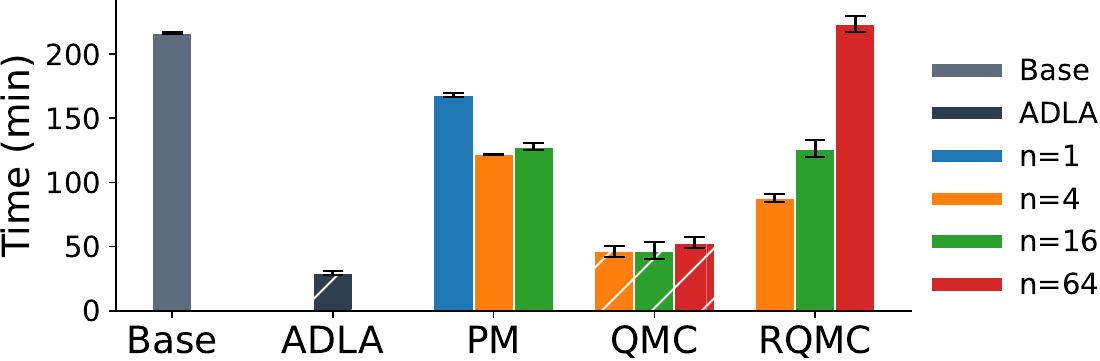}
    \includegraphics[width=0.49\linewidth]{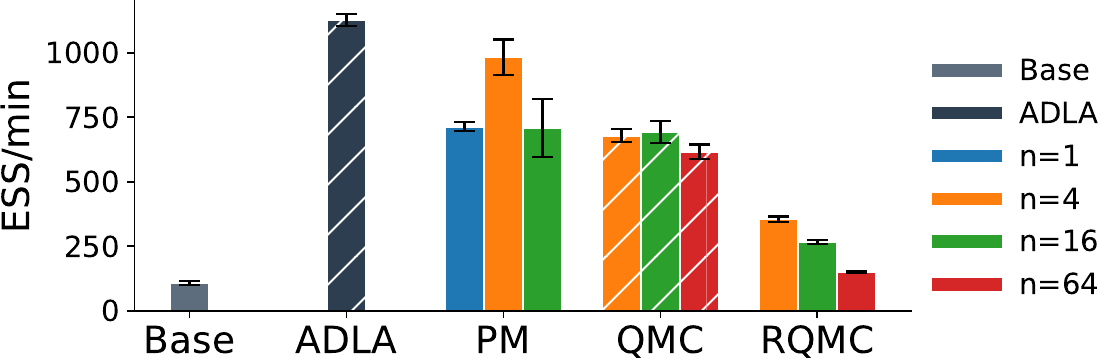}
    \caption{Running time to collect 100,000 samples and average ESS/min of $(\log\rho, \log\alpha)$ for the synthesized Gaussian process with Poisson likelihood. Hatched bars represent methods with error in the posterior. Results are collected from 5 independent runs.}
    \label{fig:gppoisson}
\end{figure*}

\subsection{Sparse kernel interaction model}
Next, we demonstrate our methods on a model applied to real data. Sparse kernel interaction models (SKIMs)~\citep{agrawal2019kernel} are an extension of generalized linear models that include interaction terms. We use the same data and model as \citet{margossian2020hamiltonian}. Details can be found in the appendix. We monitor two difficult variables in the model, $\log\tau$ and $\log\lambda_{2586}$. In this model, we find that PM is too slow to produce results in a reasonable time. The estimation error of other methods can be found in Figure \ref{fig:skim_error}. With a non-centered parameterization and NUTS, the sampling is still challenging: there are 153 divergent transitions (out of 300,000). To get full HMC work, we still need to increase the target acceptance probability to 0.99. ADLA solves this problem, but there is also an evident error in estimating the mean of the variables. With QMC, the bias can be reduced, and as we increase $n$, the error becomes smaller. The error with RQMC is smaller than with ADLA and comparable to HMC. We demonstrate this further in Figure \ref{fig:skim}. Even with $n=64$ samples, QMC still leads to error in the posterior as seen in the right tail of $\log\tau$ and left tail of $\log\lambda_{2586}$. RQMC effectively reduces this error with $n=16$ samples.
\begin{figure}
    \centering
    \includegraphics[width=\linewidth]{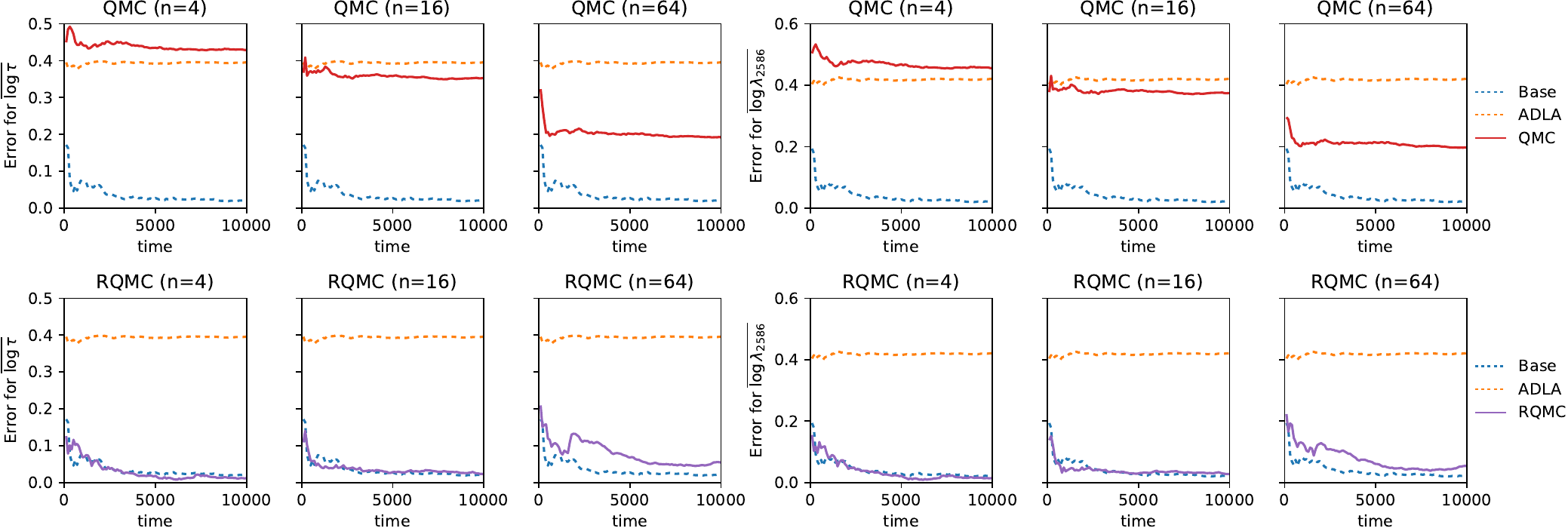}
    \caption{Error of estimating the means of parameters as a function of time in seconds for the sparse kernel interaction model. Results are averaged from 5 independent runs. Ground-truth is estimated from NUTS on the unmarginalized model. }
    \label{fig:skim_error}
\end{figure}
\begin{figure*}[t]
    \includegraphics[width=\linewidth]{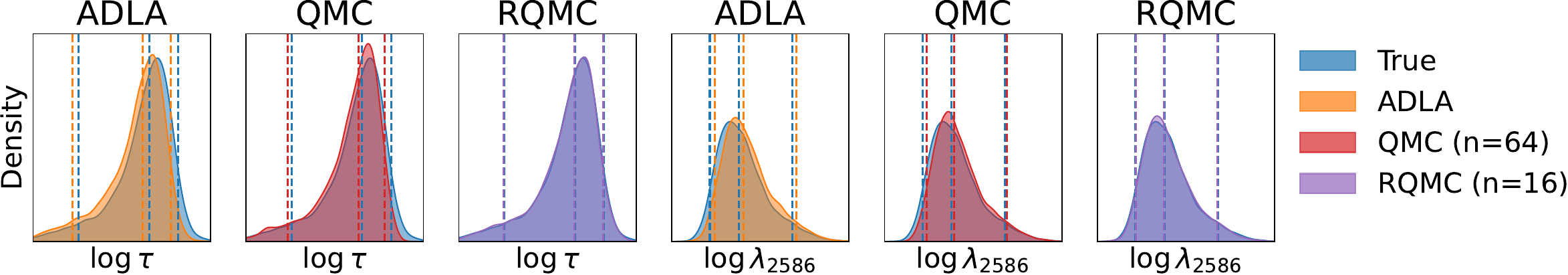}
    \caption{Comparing the sampled posterior against the ground-truth from NUTS. We demonstrate with two difficult variables: $\log\tau$ and $\log\lambda_{2586}$. Each posterior is estimated from 10,000 samples. The dashed lines are 5\%, 50\%, 95\% quantiles.}
    \label{fig:skim}
\end{figure*}
\subsection{Mixed-effects models}
Another class of LGMs are mixed-effects models, which are generalized linear models with both fixed and random effects. We use the \texttt{Epil2} dataset from glmmTMB~\citep{brooks2017glmmtmb} to demonstrate the bias reduction ability of our methods. \texttt{Epil2} studies the seizure counts for patients in a clinical trial~\citep{booth2003negative}, which has the likelihood 
\begin{align}
    y_i\sim \mathrm{NegativeBinomial}(\exp(x_i^T\beta+z_i^Tu_{g_i}),\phi)\notag
\end{align}
where for each data $i$, $y_i$ is the seizure count, $x_i$ is the fixed-effect predictor in $\mathbb{R}^6$
, $z_i$ is the random-effect predictor in $\mathbb{R}^2$
, $g_i$ is the subject index. Parameters in the model include the fixed effect $\beta$, the random effects $u_1,u_2,...,u_G$ for each subject, and the dispersion parameter $\phi$. We marginalize out the random effects for inference. More details can be found in the appendix. 

\begin{figure}[t]
    \centering
    \includegraphics[width=\linewidth]{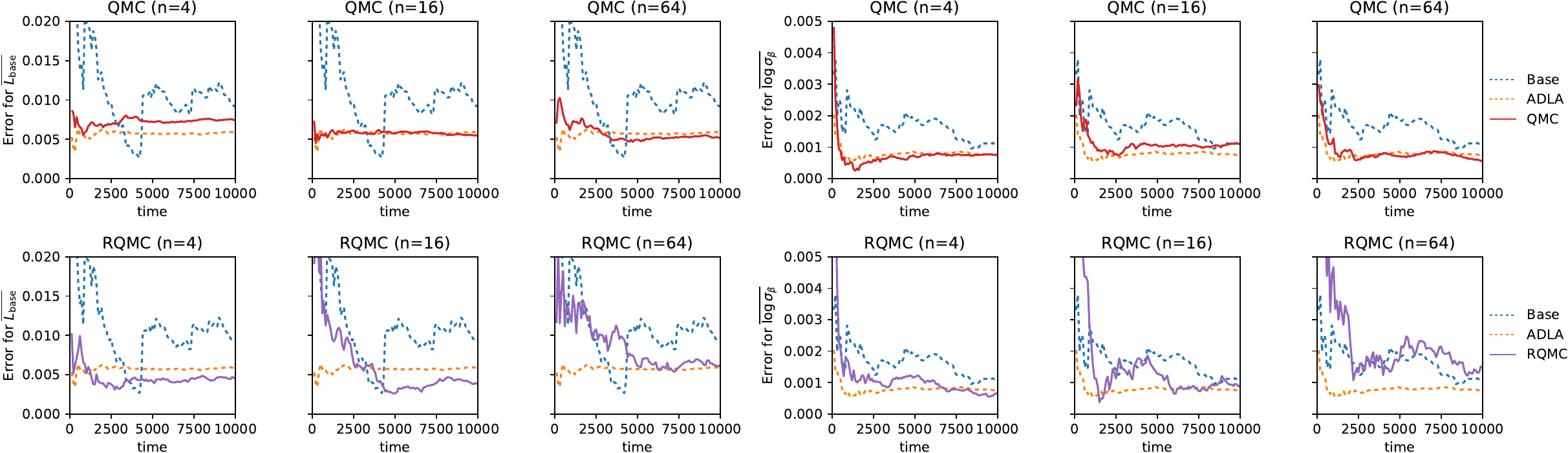}
    \caption{Error of estimating $\mathbb{E}[L_{\text{base}}]$ and $\mathbb{E}[\log\sigma_\beta]$ as a function of time in seconds for the \texttt{Epil2} model ($L_\text{base}$ is transformed from $L$ in the unconstrained space). Results are averaged from 5 independent runs. Ground-truth is estimated from NUTS with non-centered parameterization.}
    \label{fig:epil_error}
            \vspace{-10pt}
\end{figure}

For this model, we find NUTS with non-centered parameterization of random effects very efficient for inference, so we use it to get estimated ground-truths. In Figure \ref{fig:epil_error} we compare vanilla NUTS (without non-centered parameterization), ADLA, QMC, and RQMC by plotting the estimation error for means of two parameters as a function of time. In general, ADLA provides a fast estimation but has a non-vanishing estimation error. When estimating $\mathbb{E}[L_\text{base}]$, QMC with $n=64$ and RQMC with $n=4$ converge similarly as ADLA and have lower error at the end. When estimating $\mathbb{E}[\log\sigma_\beta]$, QMC is comparable with ADLA, but RQMC may give a higher error. This indicates that our methods may not always help inference and the hyperparameters should be carefully selected.

\section{Discussion}
\label{sec:dis}
We propose methods for reducing the error of the integrated Laplace approximation with an importance sampling estimator. The estimator is realized with various methods, including pseudo-marginalization, quasi-Monte Carlo and randomized quasi-Monte Carlo. Future work can explore different directions. 
First, a byproduct of importance sampling is its variance estimation, which may be used to diagnose how well the approximation is working. 
Second, in our experiments, HMC is used, but our methods can be combined with other (gradient-based) inference algorithms, such as variational inference (which has been combined with ADLA \citep{margossian25vi-symmetry}). Third, while we focus on latent Gaussian models, our methods may extend to non-Gaussian latent variables, which would require different approximate marginalization methods and different correction mechanisms. Finally, we believe it would be useful to develop high-performance implementations of our methods for the various libraries which use an integrated Laplace approximation.

\paragraph{Limitations.} We also note several limitations that may inspire future work. First, QMC-ADLA under-performs vanilla ADLA in several settings. We suspect it is because the QMC points are not symmetric around the mode, as opposed to the vanilla ADLA that is evaluated at the mode. Second, the benchmarking of different methods across models is limited. It will be interesting to test on more models with varying latent dimensions. Third, there is much space for exploration in the selection of methods and the choice of $n$ for each method. We summarized the high-level trade-offs in Table \ref{tab:alg_comparison} and presented experimental results for several choices of $n$. But an important part is analyzing importance sampling with correlated samples, such as those obtained from the reparameterization of QMC series. This is nontrivial and remains a difficult, though certainly important, research question. Fourth, we also conducted experiments on the first two models to obtain insights on the variance of the importance weights. In these experiments, we fitted a Pareto distribution to the importance weights for each posterior sample, and found that k-hat~\citep{vehtari2024pareto} (from the Pareto fit) was on average 0.5, albeit rarely above 0.7. This suggests that the variance of the importance weight may not be finite so the $\mathcal{O}(n^{-1})$ scaling in Proposition \ref{prop:1} applies conditionally. Still, our methods can be extended to use other (possibly more elaborate) proposal distributions that approximate $\pi(z|\theta,y)$ and may lead to better importance weights. Also, for PM-ADLA and RQMC-ADLA to have the correct posterior, the only requirement is the unbiasedness of importance sampling. In practice, it is possible to have difficult importance sampling but manageable MCMC sampling as shown in our experiments, but this may require further investigation.



\section*{Acknowledgment}
The authors would like to thank the anonymous reviewers for their valuable feedback and insightful suggestions, which greatly improved the quality of this manuscript. CM acknowledges the support of the Natural Sciences and Engineering Research Council (NSERC) of Canada: RGPIN-2026-07219, DGECR-2026-00308.
\bibliography{main}

@article{margossian2020hamiltonian,
  title={Hamiltonian {M}onte {C}arlo using an adjoint-differentiated {L}aplace approximation: {B}ayesian inference for latent {G}aussian models and beyond},
  author={Margossian, Charles and Vehtari, Aki and Simpson, Daniel and Agrawal, Raj},
  journal={Advances in Neural Information Processing Systems},
  volume={33},
  pages={9086--9097},
  year={2020}
}

@article{margossian2023general,
  title={General adjoint-differentiated {L}aplace approximation},
  author={Margossian, Charles C},
  journal={arXiv preprint arXiv:2306.14976},
  year={2023}
}

@inproceedings{gorinova2020automatic,
  title={Automatic reparameterisation of probabilistic programs},
  author={Gorinova, Maria and Moore, Dave and Hoffman, Matthew},
  booktitle={International Conference on Machine Learning},
  pages={3648--3657},
  year={2020},
  organization={PMLR}
}

@article{papaspiliopoulos2007general,
  title={A general framework for the parametrization of hierarchical models},
  author={Papaspiliopoulos, Omiros and Roberts, Gareth O and Sk{\"o}ld, Martin},
  journal={Statistical Science},
  pages={59--73},
  year={2007},
  publisher={JSTOR}
}

@article{tierney1986accurate,
  title={Accurate approximations for posterior moments and marginal densities},
  author={Tierney, Luke and Kadane, Joseph B},
  journal={Journal of the American statistical Association},
  volume={81},
  number={393},
  pages={82--86},
  year={1986},
  publisher={Taylor \& Francis}
}

@article{rue2009approximate,
  title={Approximate {B}ayesian inference for latent {G}aussian models by using integrated nested {L}aplace approximations},
  author={Rue, H{\aa}vard and Martino, Sara and Chopin, Nicolas},
  journal={Journal of the Royal Statistical Society Series B: Statistical Methodology},
  volume={71},
  number={2},
  pages={319--392},
  year={2009},
  publisher={Oxford University Press}
}

@book{williams2006gaussian,
  title={Gaussian processes for machine learning},
  author={Williams, Christopher KI and Rasmussen, Carl Edward},
  volume={2},
  number={3},
  year={2006},
  publisher={MIT press Cambridge, MA}
}

@article{nelder1972generalized,
  title={Generalized linear models},
  author={Nelder, John Ashworth and Wedderburn, Robert WM},
  journal={Journal of the Royal Statistical Society Series A: Statistics in Society},
  volume={135},
  number={3},
  pages={370--384},
  year={1972},
  publisher={Oxford University Press}
}

@article{caflisch1998monte,
  title={Monte {C}arlo and quasi-{M}onte {C}arlo methods},
  author={Caflisch, Russel E},
  journal={Acta Numerica},
  volume={7},
  pages={1--49},
  year={1998},
  publisher={Cambridge University Press}
}

@article{andrieu2009pseudo,
   title={The pseudo-marginal approach for efficient {M}onte {C}arlo computations},
   volume={37},
   ISSN={0090-5364},
   number={2},
   journal={The Annals of Statistics},
   publisher={Institute of Mathematical Statistics},
   author={Andrieu, Christophe and Roberts, Gareth O.},
   year={2009},
   month=Apr 
}

@inproceedings{owen2000monte,
  title={Monte {C}arlo, quasi-{M}onte {C}arlo, and randomized quasi-{M}onte {C}arlo},
  author={Owen, Art B},
  booktitle={Monte-Carlo and Quasi-Monte Carlo Methods 1998: Proceedings of a Conference held at the Claremont Graduate University, Claremont, California, USA, June 22--26, 1998},
  pages={86--97},
  year={2000},
  organization={Springer}
}

@article{gilks1995adaptive,
  title={Adaptive rejection {M}etropolis sampling within {G}ibbs sampling},
  author={Gilks, Wally R and Best, Nicky G and Tan, Keith KC},
  journal={Journal of the Royal Statistical Society Series C: Applied Statistics},
  volume={44},
  number={4},
  pages={455--472},
  year={1995},
  publisher={Oxford University Press}
}

@article{domke2018importance,
  title={Importance weighting and variational inference},
  author={Domke, Justin and Sheldon, Daniel R},
  journal={Advances in Neural Information Processing Systems},
  volume={31},
  year={2018}
}

@article{domke2019divide,
  title={Divide and couple: Using {M}onte {C}arlo variational objectives for posterior approximation},
  author={Domke, Justin and Sheldon, Daniel R},
  journal={Advances in Neural Information Processing Systems},
  volume={32},
  year={2019}
}

@article{vehtari2024pareto,
  title={Pareto smoothed importance sampling},
  author={Vehtari, Aki and Simpson, Daniel and Gelman, Andrew and Yao, Yuling and Gabry, Jonah},
  journal={Journal of Machine Learning Research},
  volume={25},
  number={72},
  pages={1--58},
  year={2024}
}

@misc{cabezas2024blackjax,
      title={BlackJAX: Composable {B}ayesian inference in {JAX}},
      author={Alberto Cabezas and Adrien Corenflos and Junpeng Lao and Rémi Louf},
      year={2024},
      eprint={2402.10797},
      archivePrefix={arXiv},
      primaryClass={cs.MS}
}

@article{hoffman2014no,
  author    = {Matthew D Hoffman and
               Andrew Gelman},
  title     = {{The No-U-turn} sampler: adaptively setting path lengths in {Hamiltonian
               Monte Carlo}},
  journal   = {Journal of Machine Learning Research},
  volume    = {15},
  number    = {1},
  pages     = {1593--1623},
  year      = {2014},
}

@article{bingham2019pyro,
  title={Pyro: Deep universal probabilistic programming},
  author={Bingham, Eli and Chen, Jonathan P and Jankowiak, Martin and Obermeyer, Fritz and Pradhan, Neeraj and Karaletsos, Theofanis and Singh, Rohit and Szerlip, Paul and Horsfall, Paul and Goodman, Noah D},
  journal={Journal of Machine Learning Research},
  volume={20},
  number={28},
  pages={1--6},
  year={2019}
}

@article{phan2019composable,
  title={Composable effects for flexible and accelerated probabilistic programming in {NumPyro}},
  author={Phan, Du and Pradhan, Neeraj and Jankowiak, Martin},
  journal={arXiv preprint arXiv:1912.11554},
  year={2019}
}

@article{kristensen2016tmb,
  title={{TMB}: automatic differentiation and {L}aplace approximation},
  author={Kristensen, Kasper and Nielsen, Anders and Berg, Casper W and Skaug, Hans and Bell, Bradley M},
  journal={Journal of Statistical Software},
  volume={70},
  pages={1--21},
  year={2016}
}

@article{monnahan2018no,
  title={{No-U-turn} sampling for fast {B}ayesian inference in {ADMB} and {TMB}: Introducing the adnuts and tmbstan {R} packages},
  author={Monnahan, Cole C and Kristensen, Kasper},
  journal={PloS one},
  volume={13},
  number={5},
  pages={e0197954},
  year={2018},
  publisher={Public Library of Science San Francisco, CA USA}
}

@article{vanhatalo2013gpstuff,
  title={{GPstuff}: {B}ayesian modeling with {G}aussian processes},
  author={Vanhatalo, Jarno and Riihim{\"a}ki, Jaakko and Hartikainen, Jouni and Jyl{\"a}nki, Pasi and Tolvanen, Ville and Vehtari, Aki},
  journal={The Journal of Machine Learning Research},
  volume={14},
  number={1},
  pages={1175--1179},
  year={2013},
  publisher={JMLR. org}
}

@Article{brooks2017glmmtmb,
    author = {Mollie E. Brooks and Kasper Kristensen and Koen J. {van
      Benthem} and Arni Magnusson and Casper W. Berg and Anders Nielsen
      and Hans J. Skaug and Martin Maechler and Benjamin M. Bolker},
    title = {{glmmTMB} Balances Speed and Flexibility Among Packages
      for Zero-inflated Generalized Linear Mixed Modeling},
    year = {2017},
    journal = {The R Journal},
    doi = {10.32614/RJ-2017-066},
    pages = {378--400},
    volume = {9},
    number = {2},
  }

@article{parno2018transport,
  title={Transport map accelerated {M}arkov chain {M}onte {C}arlo},
  author={Parno, Matthew D and Marzouk, Youssef M},
  journal={SIAM/ASA Journal on Uncertainty Quantification},
  volume={6},
  number={2},
  pages={645--682},
  year={2018},
  publisher={SIAM}
}

@article{hoffman2019neutra,
  title={Neu{T}ra-lizing bad geometry in {H}amiltonian {M}onte {C}arlo using neural transport},
  author={Hoffman, Matthew and Sountsov, Pavel and Dillon, Joshua V and Langmore, Ian and Tran, Dustin and Vasudevan, Srinivas},
  journal={arXiv preprint arXiv:1903.03704},
  year={2019}
}

@article{lai2024hamiltonian,
  title={Hamiltonian {M}onte {C}arlo inference of marginalized linear mixed-effects models},
  author={Lai, Jinlin and Domke, Justin and Sheldon, Daniel R},
  journal={Advances in Neural Information Processing Systems},
  volume={37},
  pages={29435--29463},
  year={2024}
}

@inproceedings{komodel,
  title={Model-Informed Flows for {B}ayesian Inference},
  author={Ko, Joohwan and Domke, Justin},
  booktitle={The Thirty-ninth Annual Conference on Neural Information Processing Systems},
  year={2025}
}

@article{giordano2024black,
  title={Black box variational inference with a deterministic objective: Faster, more accurate, and even more black box},
  author={Giordano, Ryan and Ingram, Martin and Broderick, Tamara},
  journal={Journal of Machine Learning Research},
  volume={25},
  number={18},
  pages={1--39},
  year={2024}
}

@InProceedings{pmlr-v244-burroni24a,
  title = 	 {Sample Average Approximation for Black-Box Variational Inference},
  author =       {Burroni, Javier and Domke, Justin and Sheldon, Daniel},
  booktitle = 	 {Proceedings of the Fortieth Conference on Uncertainty in Artificial Intelligence},
  pages = 	 {471--498},
  year = 	 {2024},
  volume = 	 {244},
  series = 	 {Proceedings of Machine Learning Research},
  month = 	 {15--19 Jul},
  publisher =    {PMLR},
}

@article{carpenter2017stan,
  title={Stan: A probabilistic programming language},
  author={Carpenter, Bob and Gelman, Andrew and Hoffman, Matthew D and Lee, Daniel and Goodrich, Ben and Betancourt, Michael and Brubaker, Marcus and Guo, Jiqiang and Li, Peter and Riddell, Allen},
  journal={Journal of Statistical Software},
  volume={76},
  pages={1--32},
  year={2017}
}

@inproceedings{agrawal2019kernel,
  title={The kernel interaction trick: Fast {B}ayesian discovery of pairwise interactions in high dimensions},
  author={Agrawal, Raj and Trippe, Brian and Huggins, Jonathan and Broderick, Tamara},
  booktitle={International Conference on Machine Learning},
  pages={141--150},
  year={2019},
  organization={PMLR}
}

@article{duane1987hybrid,
  title={Hybrid {M}onte {C}arlo},
  author={Duane, Simon and Kennedy, Anthony D and Pendleton, Brian J and Roweth, Duncan},
  journal={Physics letters B},
  volume={195},
  number={2},
  pages={216--222},
  year={1987},
  publisher={Elsevier}
}

@software{jax2018github,
  author = {James Bradbury and Roy Frostig and Peter Hawkins and Matthew James Johnson and Chris Leary and Dougal Maclaurin and George Necula and Adam Paszke and Jake Vander{P}las and Skye Wanderman-{M}ilne and Qiao Zhang},
  title = {{JAX}: composable transformations of {P}ython+{N}um{P}y programs},
  url = {http://github.com/jax-ml/jax},
  version = {0.3.13},
  year = {2018},
}

@article{alenlov2021pseudo,
  title={Pseudo-marginal {H}amiltonian {M}onte {C}arlo},
  author={Alenl{\"o}v, Johan and Doucet, Arnoud and Lindsten, Fredrik},
  journal={Journal of Machine Learning Research},
  volume={22},
  number={141},
  pages={1--45},
  year={2021}
}

@article{vanhatalo2010approximate,
  title={Approximate inference for disease mapping with sparse {G}aussian processes},
  author={Vanhatalo, Jarno and Pietil{\"a}inen, Ville and Vehtari, Aki},
  journal={Statistics in medicine},
  volume={29},
  number={15},
  pages={1580--1607},
  year={2010},
  publisher={Wiley Online Library}
}

@article{booth2003negative,
  title={Negative binomial loglinear mixed models},
  author={Booth, James G and Casella, George and Friedl, Herwig and Hobert, James P},
  journal={Statistical Modelling},
  volume={3},
  number={3},
  pages={179--191},
  year={2003},
  publisher={Sage Publications Sage CA: Thousand Oaks, CA}
}

@incollection{betancourt2015hierarchy,
  author    = {Michael Betancourt and Mark Girolami},
  title     = {Hamiltonian {M}onte {C}arlo for Hierarchical Models},
  booktitle = {Current Trends in Bayesian Methodology with Applications},
  pages     = {24},
  publisher = {Chapman and Hall/CRC},
  year      = {2015},
  doi       = {10.1201/b18502-5}
}

@article{rue2017inla,
  Title = {Bayesian Computing with {INLA}: A Review},
  Author = {Rue, H{\aa}vard and Riebler, Andrea and Sorbye, Sigrunn and Illian, Janine and Simson, Daniel and Lindgren, Finn},
  Year = {2017},
  Journal = {Annual Review of Statistics and its Application},
  Volume = {4},
  pages = {395 -- 421},
  doi ={https://doi.org/10.1146/annurev-statistics-060116-054045}}

@incollection{neal2012hmc,
	Author = {Neal, Radford M.},
	Publisher = {Chapman \& Hall / CRC Press},
	Booktitle = {Handbook of Markov Chain Monte Carlo},
	Title = {{MCMC} using {H}amiltonian Dynamics},
	Year = {2012}}

@inproceedings{yao18vi,
  author    = {Yao, Yuling and Vehtari, Aki and Simpson, Daniel and Gelman, Andrew},
  title     = {Yes, but Did It Work?: Evaluating Variational Inference},
  booktitle = {Proceedings of the 35th International Conference on Machine Learning},
  series    = {Proceedings of Machine Learning Research},
  volume    = {80},
  pages     = {5581--5590},
  year      = {2018},
  publisher = {PMLR}
}

@article{ferkingstad15copula,
  author  = {Ferkingstad, Egil and Rue, H{\aa}vard},
  title   = {Improving the {INLA} approach for approximate {B}ayesian inference for latent {G}aussian models},
  journal = {Electronic Journal of Statistics},
  year    = {2015},
  volume  = {9},
  number  = {2},
  pages   = {2706--2731},
  doi     = {10.1214/15-EJS1092}
}

@article{shun95taylor,
  author  = {Shun, Z. and McCullagh, P.},
  title   = {Laplace approximation of high dimensional integrals},
  journal = {Journal of the Royal Statistical Society: Series B},
  year    = {1995},
  volume  = {57},
  number  = {4},
  pages   = {749--760}
}

@article{chiuchiolo22inla,
  author  = {Chiuchiolo, Cristian and van Niekerk, Janet and Rue, H{\aa}vard},
  title   = {An Extended Simplified {Laplace} strategy for Approximate {B}ayesian inference of Latent {Gaussian} Models using {R-INLA}},
  journal = {arXiv:2203.14304},
  year    = {2022}
}

@article{berild22inla,
  author  = {Berild, Martin Outzen and Martino, Sara and G{\'o}mez-Rubio, Virgilio and Rue, H{\aa}vard},
  title   = {Importance Sampling with the Integrated Nested {L}aplace Approximation},
  journal = {Journal of Computational and Graphical Statistics},
  year    = {2022},
  volume  = {31},
  number  = {4},
  pages   = {1225--1237},
  doi     = {10.1080/10618600.2022.2067551}
}

@article{vanhatalo09gp,
  author  = {Vanhatalo, Jarno and Jyl{\"a}nki, Pasi and Vehtari, Aki},
  title   = {Gaussian Process Regression with a Student-t Likelihood},
  journal = {Advances in Neural Information Processing Systems},
  year    = {2009},
  volume  = {22},
  pages   = {1910--1918}
}

@article{margossian25vi-symmetry,
  author = {Margossian, Charles C. and Saul, Lawrence K.},
  title = {Generalized Guarantees for Variational Inference in the Presence of Even and Elliptical Symmetry},
  journal = {arXiv:2511.01064},
  year = {2025}
}

@inproceedings{owen2005multidimensional,
  title={Multidimensional variation for quasi-{M}onte {C}arlo},
  author={Owen, Art B},
  booktitle={International Conference on Statistics in honour of Professor Kai-Tai Fang’s 65th birthday},
  pages={49--74},
  year={2005},
  organization={World Scientific}
}
\bibliographystyle{plainnat}


\appendix
\newpage
\appendix
\onecolumn
\section*{\centering Appendices to ``Corrected Integrated Laplace Approximation''}
\setcounter{proposition}{1}
\section{Proof of the theories}
In this section, we restate and prove each of the propositions.

\subsection{Proof of Proposition 2}
\label{proof:pm}
\begin{proposition}
    $\hat{\pi}^{\PM}$, as defined in eq.~\eqref{eq:laplace_pm}, admits $\pi(\theta, y)$ as its marginal.
\end{proposition}
\begin{proof}
We need to show that $\int \hat{\pi}^{\PM}(\theta,\epsilon_{1:n},y)d\epsilon_{1:n}=\pi(\theta,y)$.
    \begin{align}
        \int\hat{\pi}^{\PM}(\theta,\epsilon_{1:n},y)du_{1:n}&=\int \prod_{i=1}^n\pi(\epsilon_i)\left(\frac{1}{n}\sum_{i=1}^n\frac{\pi(\theta)\pi(\T_{\theta,y}(\epsilon_i)|\theta)\pi(y|\theta,\T_{\theta,y}(\epsilon_i))}{\hat{\pi}(\T_{\theta,y}(\epsilon_i)|\theta,y)}\right)d\epsilon_{1:n}\notag\\
        &=\frac{1}{n}\sum_{i=1}^n\int \prod_{j=1}^n\pi(\epsilon_j)\frac{\pi(\theta)\pi(\T_{\theta,y}(\epsilon_i)|\theta)\pi(y|\theta,\T_{\theta,y}(\epsilon_i))}{\hat{\pi}(\T_{\theta,y}(\epsilon_i)|\theta,y)}d\epsilon_{1:n}\notag\\
       &=\frac{1}{n}\sum_{i=1}^n\int\frac{\pi(\theta)\pi(\T_{\theta,y}(\epsilon_i)|\theta)\pi(y|\theta,\T_{\theta,y}(\epsilon_i))}{\hat{\pi}(\T_{\theta,y}(\epsilon_i)|\theta,y)}\pi(\epsilon_i)d\epsilon_i\notag\\
       &=\frac{1}{n}\sum_{i=1}^n\int\frac{\pi(\theta)\pi(z_i|\theta)\pi(y|\theta,z_i)}{\hat{\pi}(z_i|\theta,y)}\hat{\pi}(z_i|\theta,y)dz_i\notag\\
       &=\frac{1}{n}\sum_{i=1}^n\pi(\theta,y)\notag\\
       &=\pi(\theta,y).
    \end{align}
\end{proof}
\subsection{Proof of Proposition 3}
\label{proof:qmc_converge}
\begin{proposition}
Let $f(u)=\frac{\pi(\theta)\pi(\T_{\theta,y}(u)|\theta)\pi(y|\theta,\T_{\theta,y}(u))}{\hat{\pi}(\T_{\theta,y}(u)|\theta,y)}$. If $f(u)$ has bounded variation~\citep{owen2005multidimensional}, and there exists a function $g(\theta)$ such that $\hat{\pi}_{u_{1:n}}^{\QMC}(\theta,y)<g(\theta)$ and $\int g(\theta)d\theta<\infty$, then $\int \left|\hat{\pi}_{u_{1:n}}^{\QMC}(\theta|y)-\pi(\theta|y)\right|d\theta=0$ as $n\to\infty$.
\end{proposition} 
\begin{proof}
    By the bounded variation of $f$, we invoke the properties of QMC, which suggests that $\hat{\pi}_{u_{1:n}}^{\QMC}(\theta,y)\to \pi(\theta,y)$ pointwise as $n\to\infty$. With dominated convergence theorem,
    \begin{align}
        \lim_{n\to\infty}\hat{\pi}_{u_{1:n}}^{\QMC}(y)= \lim_{n\to\infty}\int \hat{\pi}_{u_{1:n}}^{\QMC}(\theta,y)d\theta=\int \lim_{n\to\infty}\hat{\pi}_{u_{1:n}}^{\QMC}(\theta,y)d\theta=\int \pi(\theta,y)d\theta=\pi(y).
    \end{align}
    Therefore,
    \begin{align}
        \lim_{n\to\infty}\hat{\pi}_{u_{1:n}}^{\QMC}(\theta|y)=\lim_{n\to\infty}\frac{\hat{\pi}_{u_{1:n}}^{\QMC}(\theta,y)}{\hat{\pi}_{u_{1:n}}^{\QMC}(y)}=\frac{\lim_{n\to\infty}\hat{\pi}_{u_{1:n}}^{\QMC}(\theta,y)}{\lim_{n\to\infty}\hat{\pi}_{u_{1:n}}^{\QMC}(y)}=\frac{\pi(\theta,y)}{\pi(y)}=\pi(\theta|y).
    \end{align}
    Also, $\int \hat{\pi}_{u_{1:n}}^{\QMC}(\theta|y)d\theta=1=\int \pi(\theta|y)d\theta$, so by Scheffé's lemma, $\lim_{n\to\infty}\int \left|\hat{\pi}_{u_{1:n}}^{\QMC}(\theta|y)-\pi(\theta|y)\right|d\theta=0$.
\end{proof}

\subsection{Proof of Proposition 4}
\label{proof:rqmc}
\begin{proposition}
  $\hat{\pi}^{\RQMC}$, as defined in eq.~\eqref{eq:laplace_rqmc}, admits $\pi(\theta, y)$ as its marginal.
\end{proposition}
\begin{proof}
We need to show that $\int \hat{\pi}^{\RQMC}(\theta,U,y)dU=\pi(\theta,y)$. For each $i$, we can apply the change of variable $z_i=\SSS_i(U)$.
    \begin{align}
        \int\hat{\pi}^{\RQMC}(\theta,U,y)dU&=\int\pi(U)\left(\frac{1}{n}\sum_{i=1}^n\frac{\pi(\theta)\pi(\SSS_i(U)|\theta)\pi(y|\theta,\SSS_i(U))}{\hat{\pi}(\SSS_i(U)|\theta,y)}\right)dU\notag\\
            &=\frac{1}{n}\sum_{i=1}^n\int\pi(U)\frac{\pi(\theta)\pi(\SSS_i(U)|\theta)\pi(y|\theta,\SSS_i(U))}{\hat{\pi}(\SSS_i(U)|\theta,y)}dU\notag\\
            &=\frac{1}{n}\sum_{i=1}^n\int\frac{\pi(\theta)\pi(z_i|\theta)\pi(y|\theta,z_i)}{\hat{\pi}(z_i|\theta,y)}\hat{\pi}(z_i|\theta,y)dz_i\notag\\
         &=\int \pi(\theta)\pi(z|\theta)\pi(y|\theta,z)dz\notag\\
         &=\pi(\theta,y).
    \end{align}
\end{proof}
\subsection{Proof of Proposition 5}
\label{proof:recover}
\begin{proposition}
    If for each $i$, $z_i\sim \hat{\pi}(z|\theta,y)$, let $\hat{\pi}(\theta,z_{1:n},y)=\left(\frac{1}{n}\sum_{i=1}^nw_i\right) \hat{\pi}(z_{1:n}|\theta,y)$, then 
    \begin{align}
        \hat{\pi}(\theta,z,y)\coloneq\int  \hat{\pi}(z|z_{1:n},\theta,y)\hat{\pi}(\theta,z_{1:n},y)dz_{1:n}=\pi(\theta,z,y).\notag
    \end{align}
\end{proposition}
\begin{proof}
    \begin{align}
         \hat{\pi}(\theta,z,y)&=\int  \hat{\pi}(z|z_{1:n},\theta,y)\hat{\pi}(\theta,z_{1:n},y)dz_{1:n}\notag\\
         &=\int \frac{\sum_{i=1}^n w_i\delta_{z_i}(z)}{\sum_{i=1}^nw_i}\left(\frac{1}{n}\sum_{i=1}^nw_i\right) \hat{\pi}(z_{1:n}|\theta,y)dz_{1:n}\notag\\
         &=\frac{1}{n}\sum_{i=1}^n\int w_i\delta_{z_i}(z)\hat{\pi}(z_{1:n}|\theta,y)dz_{1:n}\notag\\
         &=\frac{1}{n}\sum_{i=1}^n\int w_i\hat{\pi}(z_{1:n}|\theta,y)|_{z_i=z}dz_{-i}\notag\\
         &=\frac{1}{n}\sum_{i=1}^n w_i\hat{\pi}(z_i|\theta,y)|_{z_i=z}\notag\\
         &=\frac{1}{n}\sum_{i=1}^n\pi(\theta,z,y)\notag\\
         &=\pi(\theta,z,y).
    \end{align}
\end{proof}
\section{Details of ADLA and our methods}
\label{sec:adla_detail}
\subsection{Newton's method}
We use the procedure in Algorithm \ref{alg:newton} to get $\z$ in Laplace approximation, following \citet{margossian2023general}. 
\begin{algorithm}[t]
\small
\caption{Newton's method}\label{alg:newton}
\begin{algorithmic}[1]
\REQUIRE Init position $z_0$
\STATE $K \gets K(\theta)$
\STATE $z\gets z_0$
\WHILE{$\|\nabla_z\log \pi(y|\theta,z)-a\|_\infty\le 10^{-4}$}
\STATE $c\gets \nabla_z\log \pi(y|\theta,z)$
\STATE $W \gets -\nabla_z\nabla_z\log \pi(y|\theta,z)$
\STATE $L \gets \mathrm{Cholesky}(I+W^{\frac{1}{2}}KW^{\frac{1}{2}})$
\STATE $b \gets Wz+c$
\STATE $a \gets b - W^{\frac{1}{2}}L^{-T}L^{-1}W^{\frac{1}{2}}Kb$
\STATE $z \gets Ka$
\ENDWHILE
\STATE $\z\gets z$
\end{algorithmic}
\end{algorithm}

Note at any $z$, 
\begin{align}
    \nabla_z(\log\pi(y|\theta,z)+\log\pi(z|\theta))=\nabla_z\log\pi(y|\theta,z)-K^{-1}z.\notag
\end{align}
So at $\z$, we have $\nabla_z\log\pi(y|\theta,\z)-K^{-1}\z=0$, which means $\z=K\nabla_z\log\pi(y|\theta,\z)$. This property will be used in deriving the gradients. In the Newton's method, we use $\|\nabla_z\log \pi(y|\theta,z)-a\|_\infty\le 10^{-4}$ as the stopping rule of the optimization.
\subsection{Derivation of automatic differentiation over $\xi$}
Our derivation of automatic differentiation follows the first case in \citet{margossian2023general} ($B=I+\W K\W$) that handles $\z$, and then generalizes it for $z_i$. In this section, we derive the gradients for $\xi$, the hyperparameters for $z$, contained in $\theta$. Given $\theta$, the target density of Laplace approximation is 
\begin{align}
    f(z)=\log\pi(z|\theta)+\log \pi(y|\theta,z).\notag
\end{align}
Define $W\coloneq-\nabla_z\nabla_z\log \pi(y|\theta,\z)$, then at the optimum $\z$, the Hessian of $f(z)$ is 
\begin{align}
    \nabla_z\nabla_zf(\z)=-K^{-1}-W.\notag
\end{align}
Define $A=(K^{-1}+W)^{-1}$, then the Laplace approximation is $\hat{\pi}(z|\theta,y)=\n(\z,A)$. 

The ADLA log likelihood is
\begin{align}
\label{eq:general_adla}
    \log \hat{\pi}_{\z}(y|\theta)&=\log \pi(\z|\theta)+\log \pi(y|\theta,\z)-\log \hat{\pi}(\z|\theta,y)\notag\\
    &=-\frac{1}{2}\z^TK\z-\frac{1}{2}\log|K|+\log \pi(y|\theta,\z)+\frac{1}{2}\log|A|\notag\\
    &=-\frac{1}{2}\z^TK\z+\log \pi(y|\theta,\z)-\frac{1}{2}\log|K||K^{-1}+W|\notag\\
    &=-\frac{1}{2}\z^TK\z+\log \pi(y|\theta,\z)-\frac{1}{2}\log|I+KW|.
\end{align}
We assume $W$ to be diagonal and denote $B=I+\W K\W$. Then we have
\begin{align}
    \log \hat{\pi}_{\z}(y|\theta)&=-\frac{1}{2}\z^TK\z+\log \pi(y|\theta,\z)-\frac{1}{2}\log|I+\W K\W|\notag\\
    &=-\frac{1}{2}\z^TK\z+\log \pi(y|\theta,\z)-\frac{1}{2}\log|B|.\notag
\end{align}
Also, $A$ can be computed from the Woodbury formula. 
\begin{align}
    A&=(K^{-1}+W)^{-1}\notag\\
    &=K-K\W (I+\W K\W)^{-1}\W K\notag\\
    &=K-K\W B^{-1}\W K.\notag
\end{align}
By processing the Cholesky decomposition $B=LL^T$, and compute $C=L^{-1}\W K$, we can get $A=K-C^TC$. 

Next, we derive the gradients. We consider the gradient for a single element $\xi_i$. $\nabla_{\xi_i}\log \hat{\pi}_{\z}(y|\theta)$ has two parts.  First, the explicit gradient (from Eq. (\ref{eq:general_adla}))
\begin{align}
    \left.\frac{\partial\log \hat{\pi}_{\z}(y|\theta)}{\partial\xi_i}\right|_{\mathrm{explicit}}&=\frac{1}{2}\z^TK^{-1}\frac{\partial{K}}{\partial\xi_i}K^{-1}\z-\frac{1}{2}\trace\left((W^{-1}+K)^{-1}\frac{\partial{K}}{\partial\xi_i}\right)\notag\\
    &=\frac{1}{2}\z^TK^{-1}\frac{\partial{K}}{\partial\xi_i}K^{-1}\z-\frac{1}{2}\trace\left(R\frac{\partial{K}}{\partial\xi_i}\right)\notag
\end{align}
where we define $R=(W^{-1}+K)^{-1}=\W B^{-1}\W$. The implicit gradient through $\hat{z}$ is 
\begin{align}
    \left.\frac{\partial\log \hat{\pi}_{\z}(y|\theta)}{\partial\xi_i}\right|_{\mathrm{implicit}}=\frac{\partial\log \hat{\pi}_{\z}(y|\theta)}{\partial\z}\cdot \frac{\partial\z}{\partial\xi_i}.\notag
\end{align}
Since $\z$ is the optimum, in the first term,
\begin{align}
    \frac{\partial\log \hat{\pi}_{\z}(y|\theta)}{\partial\z_j}&=-\frac{1}{2}\frac{\partial\log |I+KW|}{\partial\z_j}\notag\\
    &=-\frac{1}{2}\trace\left((K^{-1}+W)^{-1}\frac{\partial W}{\partial \z_j}\right).\notag
\end{align}
Therefore, 
\begin{align}
    \frac{\partial\log \hat{\pi}_{\z}(y|\theta)}{\partial\z}&=\frac{1}{2}(\mathrm{diag}((K^{-1}+W)^{-1})\nabla_z^3\log \pi(y|\theta,\z)|)^T\notag\\
    &=\frac{1}{2}(\mathrm{diag}(A)\nabla_z^3\log \pi(y|\theta,\z)|)^T\notag
\end{align}
For the second implicit term, by differentiating $\z=K\nabla_z\log \pi(y|\theta,\z)$, we get
\begin{align}
    \frac{\partial \z}{\partial \xi_i}&=\frac{\partial K}{\partial\xi_i}\nabla_z\log \pi(y|\theta,\z)+K\nabla_z^2\log \pi(y|\theta,\z)\frac{\partial\z}{\partial\xi_i}\notag\\
    &=\frac{\partial K}{\partial\xi_i}\nabla_z\log \pi(y|\theta,\z)-KW\frac{\partial\z}{\partial\xi_i}.\notag
\end{align}
Therefore, let $c=\nabla_z\log \pi(y|\theta,\z)$, we have
\begin{align}
    \frac{\partial \z}{\partial \xi_i}&=(I+KW)^{-1}\frac{\partial K}{\partial\xi_i}c\notag\\
    &=W^{-1}(W^{-1}+K)^{-1}\frac{\partial K}{\partial\xi_i}c\notag\\
    &=W^{-1}R\frac{\partial K}{\partial\xi_i}c\notag\\
    &=(KR+W^{-1}R-KR)\frac{\partial K}{\partial\xi_i}c\notag\\
    &=(I-KR)\frac{\partial K}{\partial\xi_i}c.\notag
\end{align}
In summary,
\begin{align}
    \nabla_{\xi_i}\log \hat{\pi}_{\z}(y|\theta)=&\frac{1}{2}\z^TK^{-1}\frac{\partial{K}}{\partial\xi_i}K^{-1}\z-\frac{1}{2}\trace\left(R\frac{\partial{K}}{\partial\xi_i}\right)\notag\\
    &+\frac{1}{2}(\mathrm{diag}(A)\nabla_z^3\log \pi(y|\theta,\z)|)^T(I-KR)\frac{\partial K}{\partial\xi_i}c.\notag
\end{align}
Observe that $a=K^{-1}\z$ and define $s=\frac{1}{2}\mathrm{diag}(A)\nabla_z^3\log \pi(y|\theta,\z)$, we further have
\begin{align}
    \nabla_{\xi_i}\log \hat{\pi}_{\z}(y|\theta)&=\frac{1}{2}a^T\frac{\partial{K}}{\partial\xi_i}a-\frac{1}{2}\trace\left(R\frac{\partial{K}}{\partial\xi_i}\right)+s(I-KR)\frac{\partial K}{\partial\xi_i}c\notag\\
    &=\frac{1}{2}\left<a,\frac{\partial{K}}{\partial\xi_i}a\right>_F-\frac{1}{2}\left<R,\frac{\partial{K}}{\partial\xi_i}\right>_F+\left<s-RKs,\frac{\partial K}{\partial\xi_i}c\right>_F\notag\\
    &=\frac{1}{2}\left<aa^T,\frac{\partial{K}}{\partial\xi_i}\right>_F-\frac{1}{2}\left<R,\frac{\partial{K}}{\partial\xi_i}\right>_F+\left<(s-RKs)c^T,\frac{\partial K}{\partial\xi_i}\right>_F\notag\\
    &=\left<\frac{1}{2}aa^T-\frac{1}{2}R+(s-RKs)c^T,\frac{\partial{K}}{\partial\xi_i}\right>_F.\notag
\end{align}
To get the last term, we can compute the vector Jacobian product (VJP) of $K$, with the initial tangent of $\Omega=\frac{1}{2}aa^T-\frac{1}{2}R+(s-RKs)c^T$.

In this work, we also need to compute $\nabla_\xi z_i$, which reduces to computing $\nabla_\xi \z$ and $\nabla_\xi A$. For $\z$,
\begin{align}
    \left<v, \frac{\partial \z}{\partial \xi_i}\right>_F&=\left<v,(I-KR)\frac{\partial K}{\partial\xi_i}c\right>_F\notag\\
    &=\left<(I-RK)vc^T, \frac{\partial K}{\partial\xi_i}\right>_F.\notag
\end{align}
For $A$ which is the same as $\Sigma$,
\begin{align}
    \frac{\partial A}{\partial\xi_i}&=\frac{\partial (K^{-1}+W)^{-1}}{\partial\xi_i}\notag\\
    &=-(K^{-1}+W)^{-1}\frac{\partial K^{-1}}{\partial\xi_i}(K^{-1}+W)^{-1}\notag\\
    &=(K^{-1}+W)^{-1}K^{-1}\frac{\partial K}{\partial\xi_i}K^{-1}(K^{-1}+W)^{-1}\notag\\
    &=(I+KW)^{-1}\frac{\partial K}{\partial\xi_i}(I+WK)^{-1}\notag\\
    &=(I-KR)\frac{\partial K}{\partial\xi_i}(I-RK).\notag
\end{align}
Therefore, given an arbitrary matrix $V$,
\begin{align}
    \left <V, \frac{\partial A}{\partial\xi_i}\right>_F&=\left <V, (I-KR)\frac{\partial K}{\partial\xi_i}(I-RK)\right>_F\notag\\
    &=\left <(I-RK)V(I-KR),\frac{\partial K}{\partial\xi_i}\right>_F.\notag
\end{align}
Define $D=I-RK$, then
\begin{gather}
    \left<v, \frac{\partial \z}{\partial \xi_i}\right>_F=\left<Dvc^T, \frac{\partial K}{\partial\xi_i}\right>_F,\notag\\
    \left <V, \frac{\partial A}{\partial\xi_i}\right>_F=\left <DVD^T,\frac{\partial K}{\partial\xi_i}\right>_F.\notag
\end{gather}
So the gradient computation for $z_i$ also reduces to the VJP of $K$.
\subsection{Derivation of automatic differentiation over $\eta$}
Next we derive the gradients for $\eta$, the hyperparameters for $y$, contained in $\theta$. Recall the ADLA log likelihood is
\begin{align}
\label{eq:general_adla2}
    \log \hat{\pi}_{\z}(y|\theta)&=-\frac{1}{2}\z^TK\z+\log \pi(y|\theta,\z)-\frac{1}{2}\log|I+KW|.
\end{align}
We consider the gradient for a single element $\eta_i$. $\nabla_{\eta_i}\log \hat{\pi}_{\z}(y|\theta)$ has two parts.  First, the explicit gradient is
\begin{align}
    \left.\frac{\partial\log \hat{\pi}_{\z}(y|\theta)}{\partial\eta_i}\right|_{\mathrm{explicit}}&=\nabla_{\eta_i}\log\pi(y|\theta,\z)-\frac{1}{2}\trace\left((K^{-1}+W)^{-1}\frac{\partial{W}}{\partial\eta_i}\right)\notag\\
    &=\nabla_{\eta_i}\log\pi(y|\theta,\z)-\frac{1}{2}\trace\left(A\frac{\partial{W}}{\partial\eta_i}\right)\notag
\end{align}
For the implicit gradient, we need $\frac{\partial \z}{\partial \eta_i}$, which is obtained from differentiating $\z=K\nabla_z\log \pi(y|\theta,\z)$:
\begin{align}
    \frac{\partial \z}{\partial \eta_i}&=K\frac{\partial}{\partial\eta_i}\nabla_z\log \pi(y|\theta,\z)+K\nabla^2_z\log \pi(y|\theta,\z)\frac{\partial \z}{\partial \eta_i}\notag\\
    &=K\frac{\partial}{\partial\eta_i}\nabla_z\log \pi(y|\theta,\z)-KW\frac{\partial \z}{\partial \eta_i}.\notag
\end{align}
Therefore, recall $c=\nabla_z\log \pi(y|\theta,\z)$, we have
\begin{align}
     \frac{\partial \z}{\partial \eta_i}&=(I+KW)^{-1}K\frac{\partial c}{\partial\eta_i}\notag\\
     &=(K^{-1}+W)^{-1}\frac{\partial c}{\partial\eta_i}\notag\\
     &=A\frac{\partial c}{\partial\eta_i}.\notag
\end{align}
Combined with the fact that $\frac{\partial\log \hat{\pi}_{\z}(y|\theta)}{\partial\z}=\frac{1}{2}(\mathrm{diag}(A)\nabla_z^3\log \pi(y|\theta,\z)|)^T$ and the explicit gradient, we have
\begin{align}
    \frac{\partial\log \hat{\pi}_{\z}(y|\theta)}{\partial\eta_i}&=\nabla_{\eta_i}\log\pi(y|\theta,\z)-\frac{1}{2}\trace\left(A\frac{\partial{W}}{\partial\eta_i}\right)+\frac{1}{2}(\mathrm{diag}(A)\nabla_z^3\log \pi(y|\theta,\z)|)^TA\frac{\partial c}{\partial\eta_i}\notag\\
    &=\nabla_{\eta_i}\log\pi(y|\theta,\z)-\frac{1}{2}\left<A,\frac{\partial{W}}{\partial\eta_i}\right>_F+\left<As,\frac{\partial c}{\partial\eta_i}\right>_F.\notag
\end{align}
So the gradient computation for $\eta$ also reduces to VJPs of $W$ and $c$. 

In this work, we need to compute $\frac{\partial \z}{\partial \eta_i}$ and $\frac{\partial A}{\partial \eta_i}$. For $\z$, consider an arbitrary vector $v$, then
\begin{align}
    \left<v, \frac{\partial \z}{\partial \eta_i}\right>_F&=\left<v,A\frac{\partial c}{\partial\eta_i}\right>_F\notag\\
    &=\left<Av, \frac{\partial c}{\partial\eta_i}\right>_F.\notag
\end{align}
For $A=\Sigma$, then
\begin{align}
    \frac{\partial A}{\partial\eta_i}&=\frac{\partial (K^{-1}+W)^{-1}}{\partial\eta_i}\notag\\
    &=-(K^{-1}+W)^{-1}\frac{\partial W}{\partial\eta_i}(K^{-1}+W)^{-1}\notag\\
    &=-A\frac{\partial W}{\partial\eta_i}A.\notag
\end{align}
So with a tangent $V$,
\begin{align}
    \left<V,\frac{\partial A}{\partial\eta_i}\right>_F&=\left<V,-A\frac{\partial W}{\partial\eta_i}A\right>_F\notag\\
    &=\left<-AVA, \frac{\partial W}{\partial\eta_i}\right>_F.\notag
\end{align}
Therefore, the VJP of $\hat{z}$ reduces to the VJP of $c$, and the VJP of $A$ reduces to the VJP of $W$.
\begin{algorithm}[t]
\caption{Adjoint-differentiated Laplace approximation~\citep{margossian2023general} }\label{alg:adla}
\begin{algorithmic}[1]
\REQUIRE $\z$ from the optimizer
\STATE $K \gets K(\theta)$\tikzmark{top}
\STATE $c \gets \nabla_z\log \pi(y|\theta,\z)$
\STATE $W \gets -\nabla_z\nabla_z\log \pi(y|\theta,\z)$
\STATE $L \gets \mathrm{Cholesky}(I+W^{\frac{1}{2}}KW^{\frac{1}{2}})$
\STATE $b \gets W\z+c$
\STATE $a \gets b - W^{\frac{1}{2}}L^{-T}L^{-1}W^{\frac{1}{2}}Kb$\tikzmark{bottom}
\STATE $R \gets W^{\frac{1}{2}}L^{-T}L^{-1}W^{\frac{1}{2}}$
\STATE $A \gets K-KW^{\frac{1}{2}}L^{-T}L^{-1}W^{\frac{1}{2}}K$
\STATE $s\gets \mathrm{diag}(A)\nabla_z^3\log \pi(y|\theta,\z)/2$
\STATE $\Omega\gets aa^T/2-R/2+(s-RKs)c^T$
\STATE $\nabla_\xi \hat{\pi}_{\z}(y|\theta)\gets \mathrm{VJP}_K(\xi,\Omega)$
\STATE $\nabla_\eta \hat{\pi}_{\z}(y|\theta)\gets \nabla_{\eta}\log\pi(y|\theta,\z)+\mathrm{VJP}_c(\eta,As)-\frac{1}{2}\mathrm{VJP}_W(\eta,A)$
\STATE $\nabla_\theta\hat{\pi}_{\z}(y|\theta)\gets \mathrm{CONCAT}([\nabla_\xi \hat{\pi}_{\z}(y|\theta),\nabla_\eta \hat{\pi}_{\z}(y|\theta)])$
\STATE $\nabla_\theta\hat{\pi}_{\z}(\theta,y)\gets\nabla_\theta\hat{\pi}_{\z}(y|\theta)+\nabla_\theta \log p(\theta)$
\end{algorithmic}
\end{algorithm}
\begin{algorithm}[t]
\caption{Computing $\nabla_\theta\hat{z}$ and $\nabla_\theta\Sigma$}\label{alg:other}
\begin{algorithmic}[1]
\REQUIRE Variables from Algorithm \ref{alg:adla}
\STATE $D\gets I-RK$
\STATE $\mathrm{VJP}_{\z}(\xi, v)\gets \mathrm{VJP}_K(\xi, Dvc^T)$
\STATE $\mathrm{VJP}_{\Sigma}(\xi,V)\gets \mathrm{VJP}_K(\xi, DVD^T)$
\STATE $\mathrm{VJP}_{\z}(\eta, v)\gets \mathrm{VJP}_c(\eta, Av)$
\STATE $\mathrm{VJP}_{\Sigma}(\eta,V)\gets \mathrm{VJP}_W(\eta, -AVA)$
\end{algorithmic}
\end{algorithm}
\subsection{Implementation within JAX}
As shown above, the gradient computation reduces to Frobenius inner products between a vector/matrix and a gradient. In JAX, this can be implemented via reverse-mode automatic differentiation with vector Jacobian products. We summarize the gradient computation for ADLA in Algorithm \ref{alg:adla}, and our generalization in Algorithm \ref{alg:other}. 
\subsection{Details of the inference algorithms}
\label{sec:mwg}

In RQMC-ADLA, we have the following model.
\begin{align}
    \hat{\pi}^{\RQMC}(\theta,U,y)= \pi(U)\left(\frac{1}{n}\sum_{i=1}^n\frac{\pi(\theta)\pi(\SSS_i(U)|\theta)\pi(y|\theta,\SSS_i(U))}{\hat{\pi}(\SSS_i(U)|\theta,y)}\right).\notag
\end{align}
The model is not continuous with respect to $U$, so it is not possible to use HMC for inference. Instead, we use the Metropolis-within-Gibbs (MwG) sampler. Note the dimension of $U$ is $d_z$, the sampling pipeline is as follows.
\begin{algorithm}[h]
\small
\caption{RQMC-ADLA with MwG sampler}\label{alg:rqmc}
\begin{algorithmic}[1]
\REQUIRE Init positions $\theta_0,U_0$
\STATE $\theta,U\gets \theta_0,U_0$
\FOR{$t=1,2,\ldots,T$}
\STATE $\theta\gets \text{HMCkernel}(\hat{\pi}^{\RQMC}(\cdot,U,y),\theta)$
\FOR{$d=1,2,\ldots,d_z$}
\STATE $U'\gets U$
\STATE $u\sim \mathrm{Uniform}(-0.1,0.1)$
\STATE $U'[d]\gets (U[d]+u)\% 1$
\STATE $U\gets \text{MHkernel}(\hat{\pi}^{\RQMC}(\theta,\cdot,y),U,U')$
\ENDFOR
\STATE $\theta_t,U_t\gets \theta,U$
\ENDFOR
\STATE \textbf{return} $\theta_{1:T},U_{1:T}$
\end{algorithmic}
\end{algorithm}
Here HMCkernel represents the HMC kernel to update $\theta$. The hyperparameters of this kernel are tuned during the warm-up phase with a fixed $U$. The MHkernel represents the Metropolis-Hastings step for updating $U$. We choose $\mathrm{Uniform}(-0.1,0.1)$ as the proposal distribution, but it can be replaced by any other distribution. 

\section{Model details}
In this section, we present more details of the models. 

\subsection{Synthesized Gaussian process}
In this model, we generate the dataset as follows:
\begin{gather}
    x_i\sim\begin{cases}
        \mathrm{Uniform}(0,2)&i\le 50\\
        \mathrm{Uniform}(2,8)&i>50
    \end{cases}\notag\\
    y_i\sim \begin{cases}
        \mathrm{Poisson}(\exp(\sin(2x_i)+2))&i\le 50\\
        \mathrm{Poisson}(\exp(\sin(2x_i)-2))&i>50
    \end{cases}
\end{gather}
This construction makes the inference hard: the posterior inference should adapt to both easier parts ($i\le 50$, larger $y_i$) and harder parts ($i>50$, smaller $y_i$). In the experiments of this model, we use 10,000 warmup samples before collecting samples for estimation.
\begin{tikzpicture}[overlay, remember picture]
    \draw [decoration={brace, amplitude=5pt}, decorate, thick]
        ($(top.north east) + (3.1cm, 0)$)- - ($(bottom.south east) + (0.48cm, 0)$);
    \node [right] at ($(top.east)!0.5!(bottom.east) + (1.9cm, 0)$) {\begin{tabular}{l} 
            Restored from\\
            Newton's method
        \end{tabular}};
\end{tikzpicture}
\subsection{Sparse kernel interaction model}
Following \citet{margossian2020hamiltonian}, we use the dataset from \citet{vanhatalo2010approximate} and define the model as follows.
\begin{gather}
    \lambda_i\sim\Stu(\nu_{\mathrm{local}},0,1),\ \tau\sim\Stu(\nu_{\mathrm{global}},0,s_{\mathrm{global}}),\ c_\mathrm{aux},\chi\sim \Gamma^{-1}(s_\mathrm{df}/2,s_\mathrm{df}/2),\notag\\
    c=s_\mathrm{slab}\sqrt{c_\mathrm{aux}},\ \tilde{\lambda}_i^2=c^2\lambda_i^2/(\tau^2+c^2\lambda_i^2),\ \eta_2=\tau^2\chi/c^2,\notag\\
    K_1=X\mathrm{diag}(\tilde{\lambda}^2)X^T,\ K_2=(X\circ X)\mathrm{diag}(\tilde{\lambda}^2)(X\circ X)^T\notag\\
    K=\frac{1}{2}\eta_2^2(K_1+1)\circ(K_1+1)-\frac{1}{2}\eta_2^2K_2-(\tau^2-\eta_2^2)K_1+c_0^2-\frac{1}{2}\eta_2^2\notag\\
    z\sim\MVN(0, K),\ y\sim \mathrm{Bernoulli}(\text{logits}=z),
\end{gather}
where $\mathrm{local},\mathrm{global},s_{\mathrm{global}},s_\mathrm{df},s_\mathrm{slab},c_0$ are given in the dataset, $X$ is the data matrix and $\circ$ represents the element-wise multiplication.

For this experiment, we use 200 covariates (indexed between 2500 and 2700) to construct $X$. During inference, we warm up with 10,000 samples before collecting samples. 

\subsection{Mixed-effects models}
The \texttt{Epil2} model has the likelihood 
\begin{align}
    y_i\sim \mathrm{NegativeBinomial}(\exp(x_i^T\beta+z_i^Tu_{g_i}),\phi).\notag
\end{align}
And the priors are
\begin{gather}
        \sigma_{\beta}\sim \n^+(0,1),\ T\sim\MVN^+(0,\text{diag}(1,1)),\ L\sim \mathrm{LKJ}(2,1),\notag\\
        \phi\sim\text{Exponential}(1),\ \beta_k\sim \n(0,\sigma_{\beta}^2),\ k=1,...,6,\ u_{j}\sim\MVN(0, TLL^TT),\ j=1,...,G.\notag
\end{gather}
For efficient inference, we may want to marginalize each $u_j$ from the model, as suggested in \citet{lai2024hamiltonian}. To be practical for Laplace approximation, we replace the distributions of $u_{1:G}$ and $y_i$ as
\begin{gather}
   u'\sim\MVN(0, Z\Sigma_{u_{1:G}} Z^T),\  y_i\sim \mathrm{NegativeBinomial}(\exp(x_i^T\beta+u'_i),\phi),\notag
\end{gather}
where $Z$ is the design matrix for random effects and $\Sigma_{u_{1:G}}$ is a block diagonal matrix consisting of $G$ blocks of $TLL^TT$. This model is equivalent to the previous model, but is better for Laplace approximation because the $W$ matrix in Algorithm \ref{alg:adla} becomes diagonal. During Laplace approximation, we also use the fact that $Z\Sigma_{u_{1:G}} Z^T$ is a block-diagonal matrix to speed up computation. 

\section{Additional results}
\subsection{The raw importance sampling estimator}
\label{sec:raw_is}
\begin{figure}[t]
    \centering
    \includegraphics[width=\linewidth]{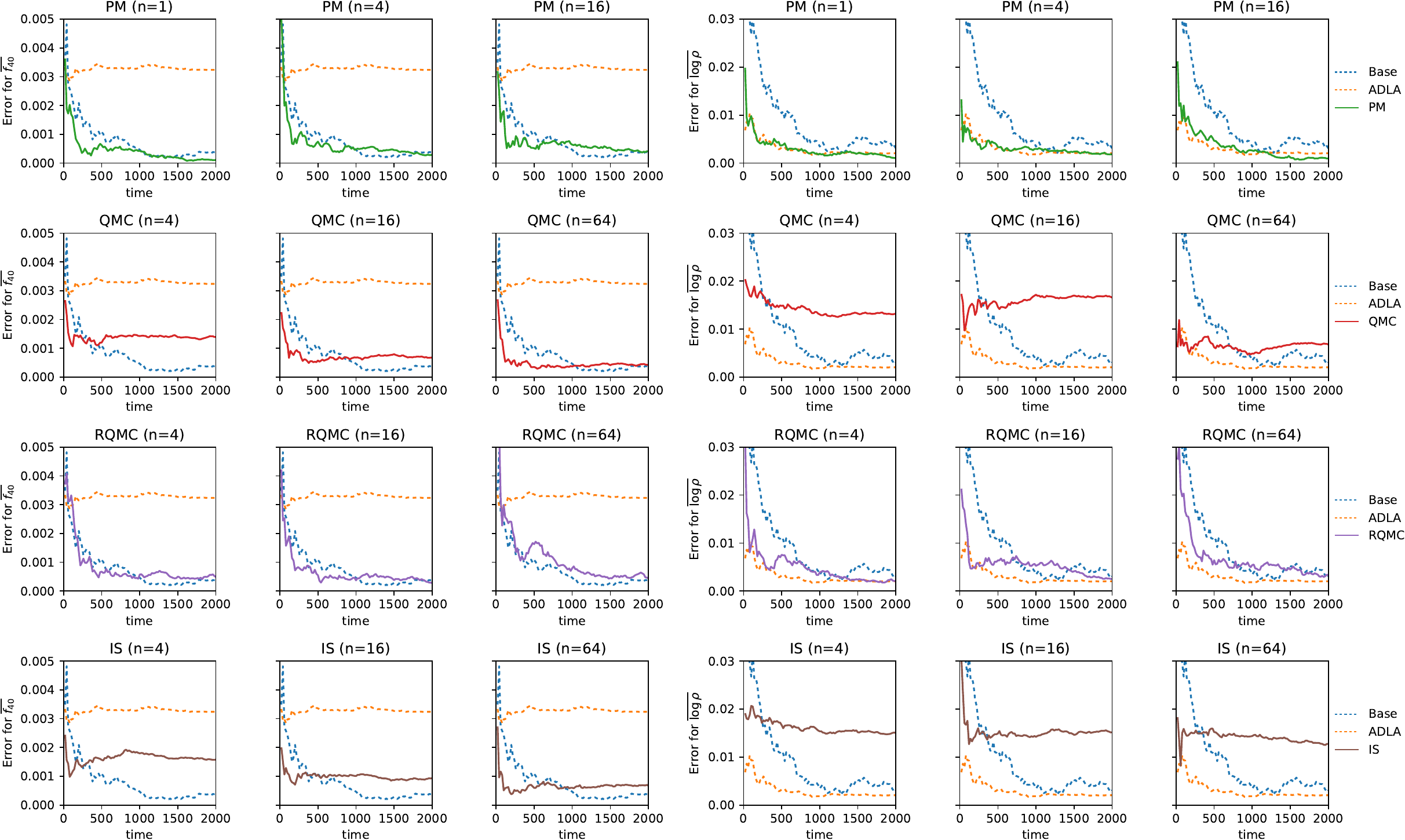}
    \caption{Error of estimating the means of parameters as a function of time in seconds for the synthesized Gaussian process with Poisson likelihood. Results are averaged from 5 independent runs. Ground-truth is estimated from NUTS on the unmarginalized model. The last row presents the results for $\hat{\pi}_{\epsilon_{1:n}}^\mathrm{IS}(\theta,y)$, and the first three rows are the same as Figure \ref{fig:gppoisson_error}. }
    \label{fig:gp_error_all}
\end{figure}
As a sanity check, we also tested the importance sampling-based model $\hat{\pi}_{\epsilon_{1:n}}^\mathrm{IS}(\theta,y)$ on the Gaussian process model. The results can be found in Figure \ref{fig:gp_error_all}. We find that it gives a higher error than QMC-ADLA, despite having the same cost. This is consistent in our other experiments, so we do not include it in our main results. 
\subsection{Additional sampling results}
We include additional sampling results of our experiments in Tables \ref{tab:sgp_additional}, \ref{tab:skim_additional} and \ref{tab:epil_additional}. In all models in our experiments, the Base method presents a non-negligible number of divergent transitions in HMC sampling. ADLA and our methods solve or mitigate this problem. 

\begin{table}[t]
    \centering
    \begin{tabular}{ccccc}
         \toprule
         Method &$n$& Number of samples& Time (min) & Divergences \\
         \midrule
         Base&-&100000& 217 (1)& 0 (0)\\
         ADLA&-&100000& 29 (2)& 0 (0)\\
         PM&1&100000& 168 (2)& 0 (0)\\
         PM&4&100000& 122 (0)& 0 (0)\\
         PM&16&100000& 128 (3)& 0 (0)\\
         QMC&4&100000& 46 (4)& 0 (0)\\
         QMC&16&100000& 47 (6)& 0 (0)\\
         QMC&64&100000& 53 (4)& 0 (0)\\
         RQMC&4&100000& 88 (3)& 0 (0)\\
         RQMC&16&100000& 126 (7)& 0 (0)\\
         RQMC&64&100000& 224 (6)& 0 (0)\\
         \bottomrule
    \end{tabular}
    \caption{More details about the experiments on the synthesized Gaussian process model. Mean and standard deviation across 5 runs are reported.}
    \label{tab:sgp_additional}
\end{table}

\begin{table}[t]
    \centering
    \begin{tabular}{ccccc}
         \toprule
         Method &$n$& Number of samples& Time (min) & Divergences \\
         \midrule
         Base&-&300000& 892 (83)& 4 (2)\\
         ADLA&-&300000& 431 (82)& 0 (0)\\
         QMC&4&100000& 465 (166)& 0 (0)\\
         QMC&16&100000& 358 (15)& 0 (0)\\
         QMC&64&100000& 401 (3)& 0 (0)\\
         RQMC&4&100000& 372 (12)& 4 (6)\\
         RQMC&16&100000& 447 (22)& 1 (1)\\
         RQMC&64&100000& 773 (153)& 3 (4)\\
         \bottomrule
    \end{tabular}
    \caption{More details about the experiments on the sparse kernel interaction model. Mean and standard deviation across 5 runs are reported.}
    \label{tab:skim_additional}
\end{table}

\begin{table}[t]
    \centering
    \begin{tabular}{ccccc}
         \toprule
         Method &$n$& Number of samples& Time (min) & Divergences \\
         \midrule
         Base&-&5000000&  442 (42) &164838 (86482)\\
         ADLA&-&4000000&328 (29) &10 (3)\\
         QMC&4&2500000& 424 (60) &5 (0)\\
         QMC&16&2500000& 543 (61) &8 (2)\\
         QMC&64&1500000&465 (43) &4 (2)\\
         RQMC&4&500000&429 (19) &2 (1) \\
         RQMC&16&100000& 490 (27) &0 (0)\\
         RQMC&64&100000&  794 (31) &0 (1)\\
         \bottomrule
    \end{tabular}
    \caption{More details about the experiments on the mixed effects models. Mean and standard deviation across 5 runs are reported.}
    \label{tab:epil_additional}
\end{table}
\subsection{More convergence results}
We selected two variables for each model to demonstrate the error convergence in the main text. In this part, we include more convergence results. For the synthesized Gaussian process model, we also estimate $\mathbb{E}[f_{80}]$ and $\mathbb{E}[\log\alpha]$ with the samples. The results are in Figure \ref{fig:sgp_error2}. We find our methods correct the error of ADLA and have faster or comparable convergence as HMC. For the sparse kernel interaction model, we estimate $\mathbb{E}[\log c_\mathrm{aux}]$ and $\mathbb{E}[\log\chi]$ and the results are in Figure \ref{fig:skim_error2}. For these two estimations, our methods correct the error but converge slower than HMC. This is mainly because the two variables are easy to sample, so the benefits of marginalization are not significant. For the mixed-effects model, we show the estimation error for $\mathbb{E}[\beta_0]$ and $\mathbb{E}[\log T_1]$ in Figure \ref{fig:epil_error2}. We find that, compared to HMC, ADLA and our methods have a higher error when estimating $\mathbb{E}[\beta_0]$ but a lower error when estimating $\mathbb{E}[\log T_1]$, although the magnitude of the first error is small. For the first error, it may be because there is still an error in the optimizer. For the second error, note that the raw HMC reports a significant number of divergent samples in Table \ref{tab:epil_additional}, which biases the estimation. 
\begin{figure}[t]
    \centering
    \includegraphics[width=\linewidth]{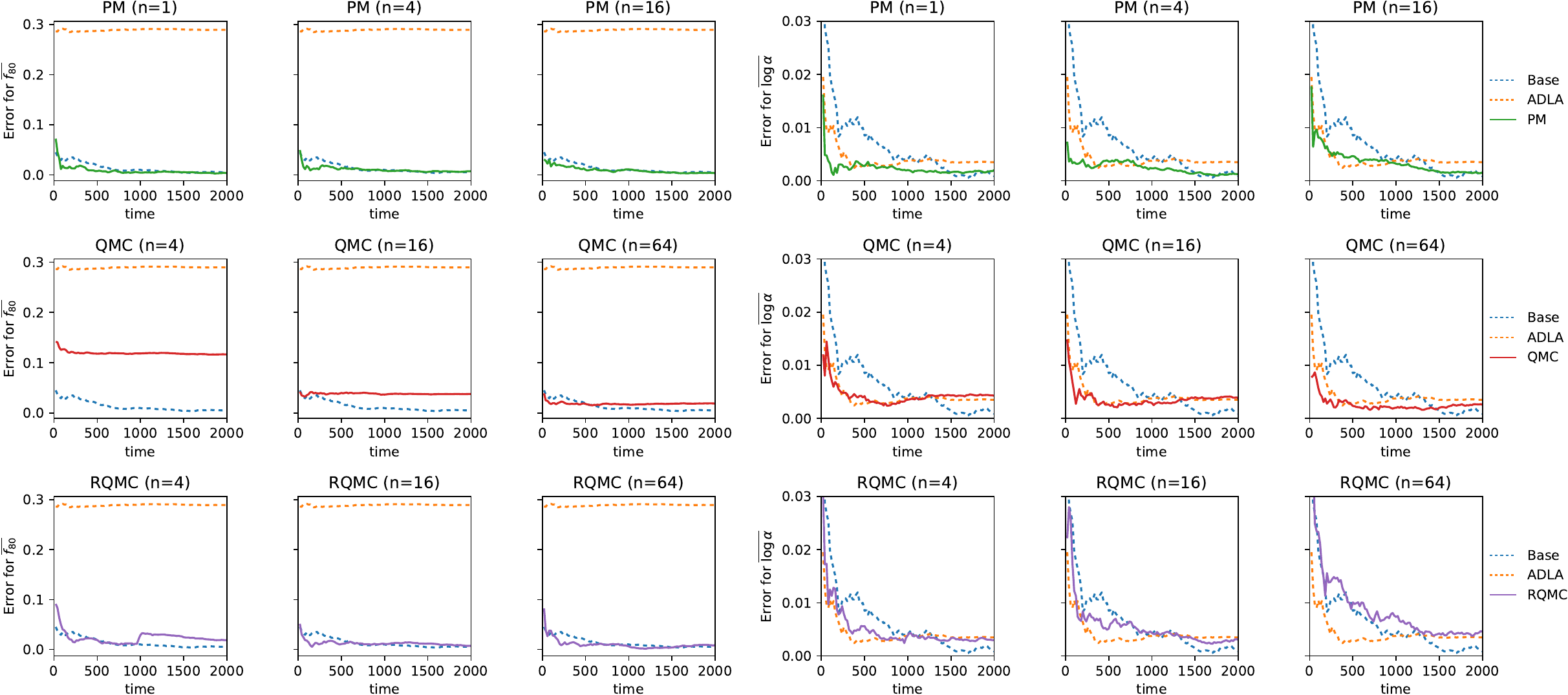}
    \caption{Error of estimating $\mathbb{E}[f_{80}]$ and $\mathbb{E}[\log\alpha]$ as a function of time in seconds for the
synthesized Gaussian process with Poisson likelihood. Results are averaged from 5 independent runs.}
    \label{fig:sgp_error2}
\end{figure}

\begin{figure}[t]
    \centering
    \includegraphics[width=\linewidth]{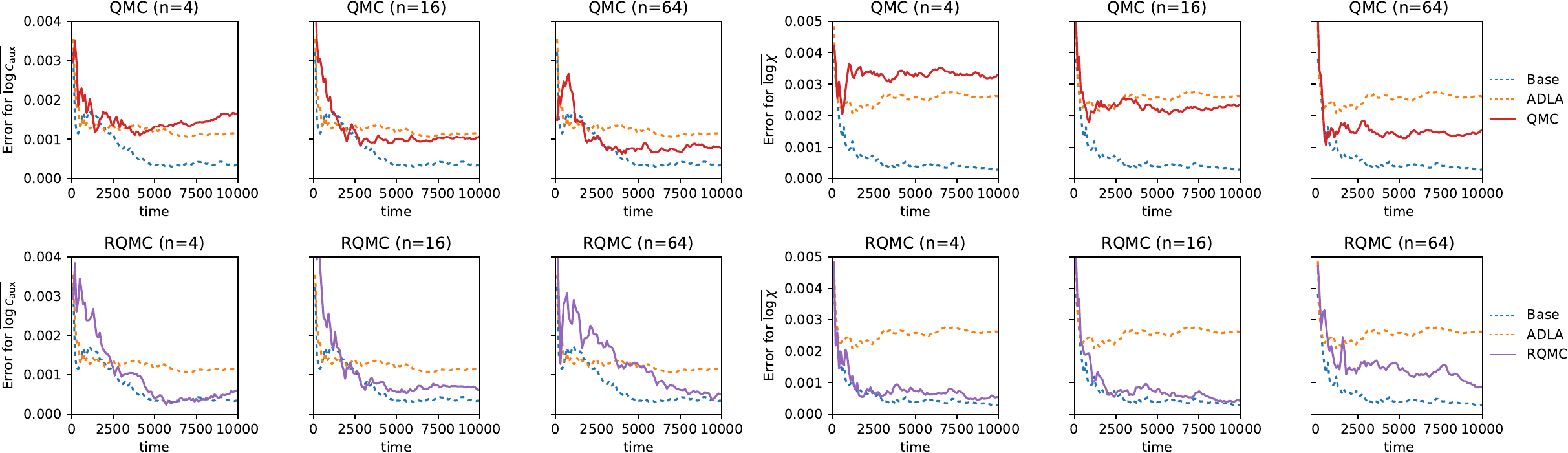}
    \caption{Error of estimating $\mathbb{E}[\log c_\mathrm{aux}]$ and $\mathbb{E}[\log\chi]$ as a function of time in seconds for the
sparse kernel interaction model. Results are averaged from 5 independent runs.}
    \label{fig:skim_error2}
\end{figure}

\begin{figure}[t]
    \centering
    \includegraphics[width=\linewidth]{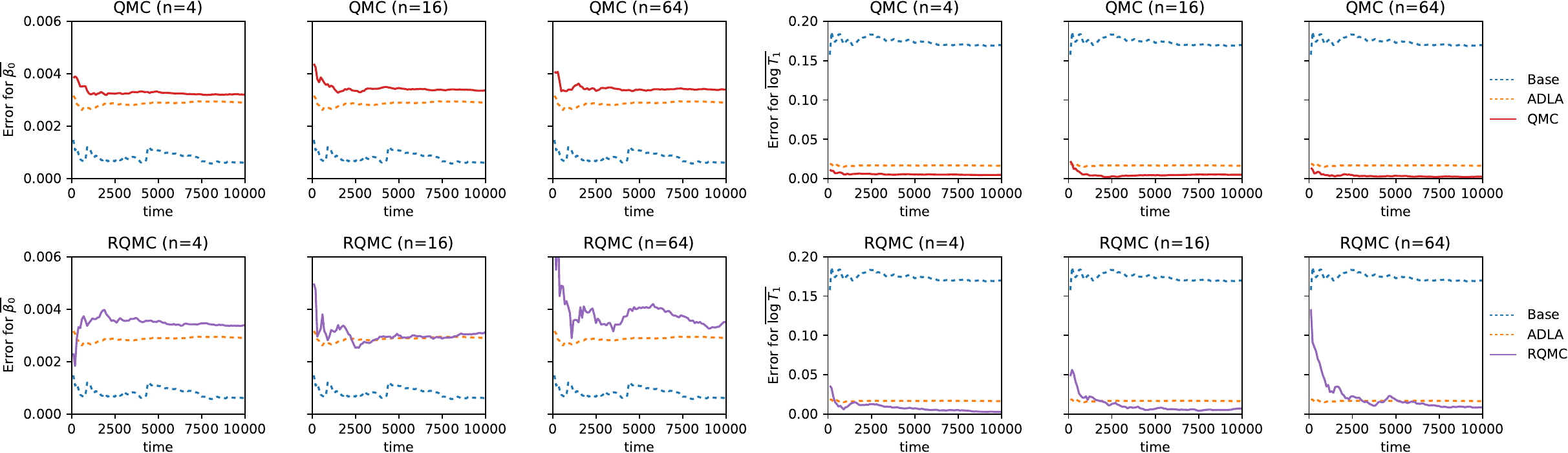}
    \caption{Error of estimating $\mathbb{E}[\beta_0]$ and $\mathbb{E}[\log T_1]$ as a function of time in seconds for the
mixed-effects model. Results are averaged from 5 independent runs.}
    \label{fig:epil_error2}
\end{figure}
\subsection{The posterior of RQMC-ADLA}
In RQMC-ADLA, an issue is that $\hat{\pi}^\mathrm{RQMC}(\theta,U,y)$ is not continuous with $U$. Here we demonstrate this problem with a simple example model:
\begin{align}
    \theta\sim\n(0,1),\ z\sim\n(\theta,1),\ y\sim \text{Cauchy}(z,1),
\end{align}
where the observation is $y=2$. We can compute $\hat{\pi}^\mathrm{RQMC}(U|\theta,y)$ with numerical integration. In Figure \ref{fig:rqmc_example}, we see that the discontinuity exists in this model. The number of discontinuous points is the same as $n$. Compared with the uncountable number of real numbers in $[0,1]$, the probability of sampling at a discontinuous point remains 0 even with large $n$. Also, as we increase $n$, $\hat{\pi}^\mathrm{RQMC}(U|\theta,y)$ becomes more uniform, implying easier inference. 
\begin{figure}
    \centering
    \includegraphics[width=\linewidth]{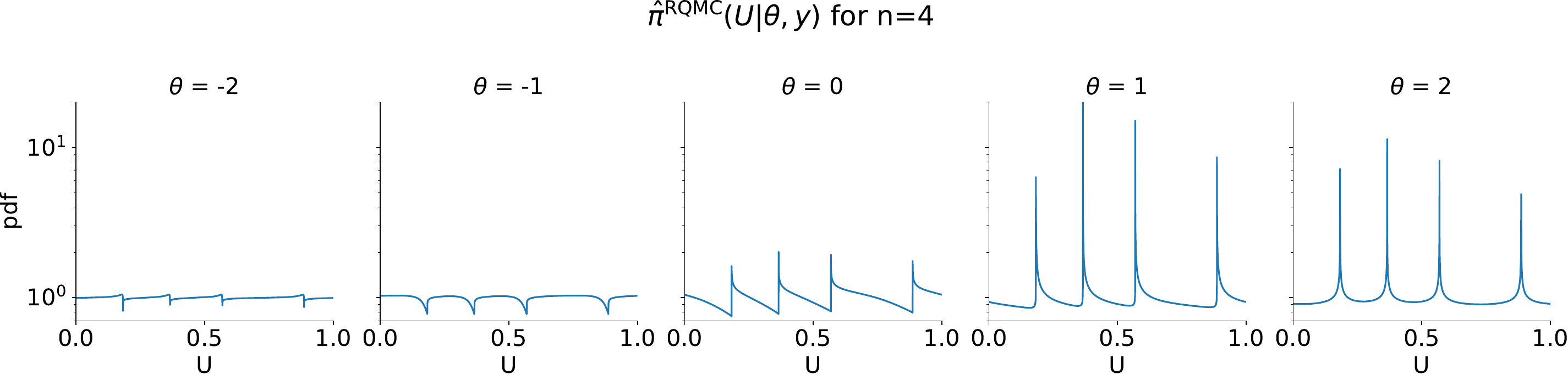}
    \includegraphics[width=\linewidth]{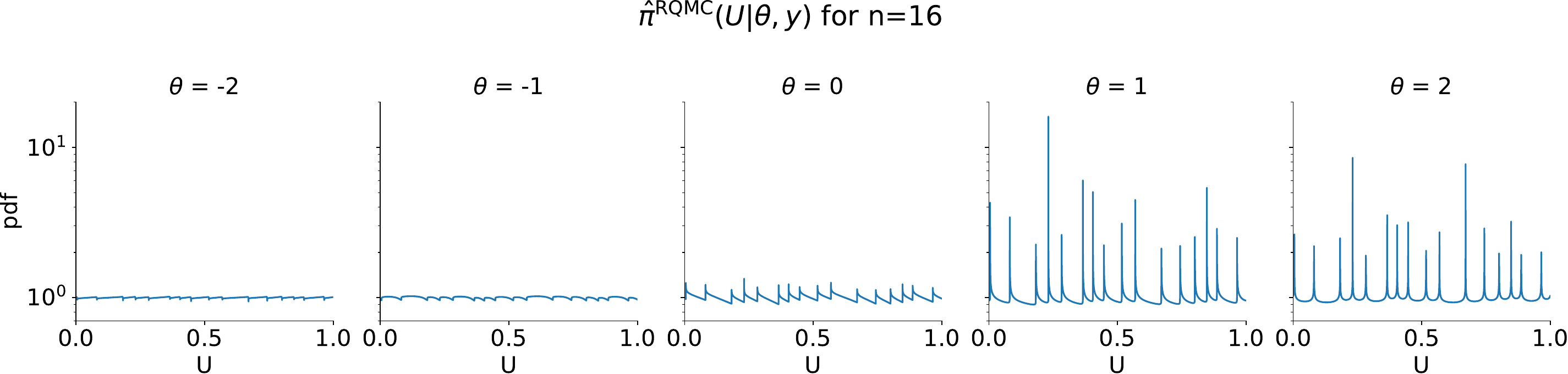}
    \includegraphics[width=\linewidth]{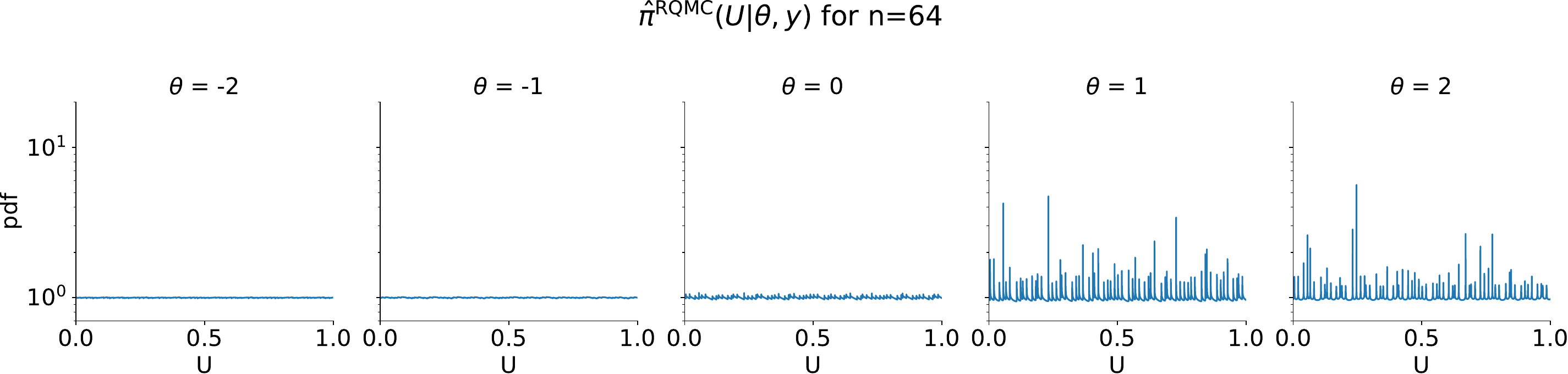}
    
    \caption{Distribution of $\hat{\pi}^\mathrm{RQMC}(U|\theta,y)$ with different $\theta$ and $n$ for the simple example model, using 5,000 evaluation points in $[0,1]$. The distribution is more uniform as we increase $n$. Also, $\hat{\pi}^\mathrm{RQMC}(U|\theta,y)$ is not continuous.}
    \label{fig:rqmc_example}
\end{figure}

\end{document}